\documentclass[11pt]{article}

\usepackage{times}
\usepackage{hyperref}

\usepackage{amsfonts}
\usepackage{wrapfig}
\usepackage{amsbsy}
\usepackage{amsmath}
\usepackage{tikz}
\usepackage{graphicx}
\usepackage{amssymb}
\usepackage{mdframed}
\usepackage{bbm}
\usepackage{enumerate}
\usepackage[lined]{algorithm2e}
\usepackage{tikz-qtree}
\usepackage[titletoc]{appendix}
\usepackage{subcaption}
\usepackage{mwe}
\usepackage[export]{adjustbox}

\usepackage{listings}
\usepackage{bbm}
\usepackage{framed}
\usepackage{algorithmic}
\usepackage{mathtools}

\allowdisplaybreaks

\newcommand{\pr}[1]{{\rm Pr}\left[{#1}\right]}

\newcommand{\algref}[1]{Fig.~\ref{#1}}
\newcommand{\secref}[1]{Sec.~\ref{#1}}
\newcommand{\figref}[1]{Fig.~\ref{#1}}

\newcommand{\thmref}[1]{Thm.~\ref{#1}}

\newcommand{\remref}[1]{Remark~\ref{#1}}

\newcommand{\appref}[1]{App.~\ref{#1}}
\newcommand{\lemref}[1]{Lem.~\ref{#1}}

\providecommand{\BR}[1]{\mathrm{\mathbf{#1}}}
\providecommand{\E}{\mathbb{E}}

\providecommand{\para}[1] {\left( {#1} \right)}

\providecommand{\amt}{\nu}
\providecommand{\lb}{\underline{\nu}}
\providecommand{\lbf}{\underline{\nu}^\mathrm{d}}
\providecommand{\lbs}{\underline{\nu}^\mathrm{p}}
\providecommand{\mo}{\beta}

\providecommand{\alloc}{M}
\providecommand{\scc}{X}
\providecommand{\nit}{n}

\providecommand{\err}{\epsilon}
\providecommand{\regret}{R}

\providecommand{\verr}{\boldsymbol{\varepsilon}}
\providecommand{\mulc}{c}

\providecommand{\nban}{K}
\providecommand{\nal}{\ell}
\providecommand{\rem}{S^*}
\providecommand{\indi}{\boldsymbol{1}}
\providecommand{\add}{r}
\providecommand{\rv}[1]{\boldsymbol{#1}}
\providecommand{\rvt}{\rv{\tau}}
\providecommand{\sas}{s^{\mathrm{p}}}
\providecommand{\saf}{s^{\mathrm{d}}}
\providecommand{\xs}{x^{\mathrm{p}}}

\newcommand{\figline}{\rule{0.50\textwidth}{0.5pt}}

\usepackage{natbib}
\usepackage{amsthm}
\usepackage{fullpage}
\newtheorem{theorem}{Theorem}
\newtheorem{lemma}{Lemma}
\newtheorem{remark}{Remark}
\newtheorem{proposition}{Proposition}

\author{Yuval Dagan\thanks{Technion - Israel Institute of Technology, \texttt{yuvaldag@gmail.com }}
	\and
	Koby Crammer\thanks{Technion - Israel Institute of Technology, \texttt{koby@ee.technion.ac.il}}}

\title{}

\begin{document}
\title{A Better Resource Allocation Algorithm with Semi-Bandit Feedback}


\maketitle

\begin{abstract}
  We study a sequential resource allocation problem between a fixed
  number of arms. On each iteration the algorithm distributes a
  resource among the arms in order to maximize the expected success
  rate. Allocating more of the resource to a given arm increases the
  probability that it succeeds, yet with a cut-off. We follow
  \cite{LCS} and assume that the probability increases
  linearly until it equals one, after which allocating more of the
  resource is wasteful. These cut-off values are fixed and unknown to
  the learner. We present an algorithm for this problem and
  prove a regret upper bound of $O(\log n)$ improving over the best
  known bound of $O(\log^2 n)$. Lower bounds we prove show that our
  upper bound is tight. Simulations demonstrate the superiority of our
  algorithm.
\end{abstract}


\section{Introduction}
We study a sequential resource allocation problem for a fixed number
of arms (or processes). On each iteration $t$, the learning algorithm
distributes a fixed amount of unit resource between $K$ arms, and pulls
all the arms. The probability of each arm to succeed is proportional
to the amount of resource assigned to it (or $1$, if enough resource
was assigned), with slope that depends on the arm, and unknown to
the learner. The learner observes the result of all arms, and repeats
the process. Her goal is to maximize the cumulative number of successes
over all $K$ arms and all $n$ iterations.

Formally, on time $t$ the learner assigns $M_{k,t} \geq 0$ resource
for arm $k = 1 \dots K$, such that $\sum_{k=1}^K M_{k,t} \leq 1$. The
outcome of the allocation processes is $\scc_{k,t}=1$ if arm $k$
succeeded and $\scc_{k,t}=0$ if it fails. The probability of arm $k$
to succeed given $\alloc_{k,t}$ is $\pr{\scc_{k,t}=1 \mid \alloc_{k,t}}=\min\{1, M_{k,t} / \nu_k\}$ for some
fixed unknown values $\nu_1 \dots \nu_K$. The goal of the learner is
to maximize $\sum_{k,t} \scc_{k,t}$.

The problem was first suggested by \cite{LCS}, who proposed an
algorithm and a corresponding regret bound inspired by the \emph{upper
  confidence interval} (UCB) algorithm of \cite{UCB} for the
stochastic multi-armed bandit problem.  The algorithm of \cite{LCS} maintains high probability lower bound estimates on
the parameters $\amt_1, \dots, \amt_\nban$.  On every iteration $t$,
the arms are prioritized by these bounds, from the lowest to the
highest, each arm getting an amount of resources which equals its
lower bound, until no resource is left.  Using this technique, the
best arms get almost all the resource they require, hence, their
probability of success is close to $1$, and their outcomes have a low
variance.  This enables the authors to estimate $\amt_k$ with an
expected error of $\Theta\left( \frac{1}{t} \right)$ after $t$
allocations.  Yet, the proof requires the constructed lower bound
estimates to hold throughout all the $n$ iterations, which implies
that their failure probability has to be low. This high confidence
requirement weakens the tightness of this
estimate: it is far by $\Theta(\log n/t)$ from the estimated parameter,
yielding a regret of $O(\log^2 n)$.

We propose a new algorithm that utilizes both probabilistic lower
bounds and deterministic lower bound estimates, utilizing the fact that the
error is one-sided: if arm $k$ is allocated with $\alloc$
resources and terminates in failure, we know that $\amt_k > \alloc$
with probability $1$. We analyze this algorithm and prove a regret of
$O(\log n)$. Besides having a lower regret bound than \cite{LCS},
our algorithm does not have to know the
horizon $n$ in advance (without using a doubling trick).
Simulations we performed demonstrate the superiority of our algorithm
(by a considerable gap),
and a matching $\Omega(\log n)$ lower bound is obtained.

This problem is a special case of stochastic partial monitoring
problems, first studied by \cite{RAT89}. These are exploration
vs. exploitation problems, where the user performs actions and obtains
a stochastic reward based on them, and on an additional hidden
parameter.  \cite{LCS} surveyed relevant literature on this topic,
including the work by \cite{BM12}.  The model discussed in our paper was
generalized by \cite{LCS2}, to enable multiple resource types.  They
discuss the relation to stochastic linear bandits~\citep{AYSC11,AR13}
and online combinatorial optimization~\citep{KBW15}.

\section{Single Arm Problem}
\label{sec:single}

We start our discussion in a setting with only a {\em single} arm. On
each iteration $1 \le t \le n$ an algorithm assigns some amount
$\alloc_t \geq 0$ ׂof a resource to the arm and pulls that arm. It then
obtains an indication of success (denoted by $\scc_t = 1$) or failure
($\scc_t = 0$).  The arm is associated with a threshold parameter
$\amt$ such that the probability of success given an allocation of
$M_t$ equals $\min(1, \alloc_t / \amt)$, as in the multi-armed
setting.  Each allocation incurs a cost of $\alloc_t$, and the total reward on
iteration $t$ equals $\scc_t - \alloc_t$.

\figref{fig:reward_single} illustrates the expected reward as a function
of the allocated amount: it is a piecewise linear function, maximized
at $\alloc_t = \amt$, with a reward of $1 - \amt$.  
The regret of the
algorithm on iteration $t$ %
is defined as the difference between the
maximal expected reward, and the actual reward,
\[
	\regret_t = 1 - \amt - (\scc_t - \alloc_t),
\]
and the total regret equals $\regret^{(n)} = \sum_{t=1}^n \regret_t$.

\algref{alg:one-arm} summarizes our algorithm for the single-arm resource
allocation problem, that invokes the arm for $n$ rounds, when $n$ (and
of course $\amt$) are unknown in advance. 
The algorithm maintains a guaranteed
lower bound on $\amt$. On each iteration it allocates a slightly
higher amount of resource than the lower bound.  If the machine
fails, the amount of resource which was allocated is
insufficient, and the lower bound is increased.  Its new value is set as
the amount of resource allocated just before failure.

Specifically, the lower bound is initialized to $\lb_0 \gets 0$. 
On iteration $t=1\dots$ the  algorithm allocates
$\alloc_t \gets \lb_{t-1} + \frac{2}{t}$.
After pulling the arm and observing $\scc_t$
the algorithm increases the current lower bound and sets
$\lb_t \gets \alloc_t$ after failure ($\scc_t=0$)
and does not modify the lower bound after a success ($\scc_t=1$),
that is, $\lb_t \gets \lb_{t-1}$.

\begin{figure}
\begin{minipage}{.5\textwidth}
	\begin{center}
		\includegraphics[width=0.9\textwidth]{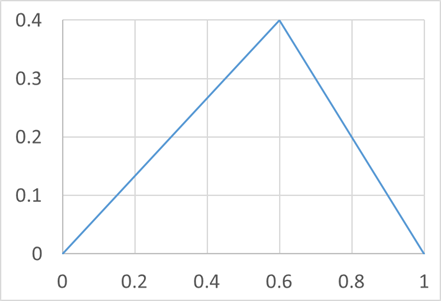}    
		\caption{{\footnotesize Reward as a function of the resource for $\nu=0.6$}}\label{fig:reward_single}
	\end{center}
\end{minipage}
\begin{minipage}{.5\textwidth}
	\begin{center}
			\begin{algorithmic}[1]
				\STATE$\lb_0 \gets 0$
				\FORALL {$t \in \{1, 2, \dots\}$}
				\STATE $\alloc_t \gets \lb_{t-1} + \frac{2}{t}$
				\STATE Grant the arm with $\alloc_t$ resources
				\STATE $\scc_t \gets$ success status of the arm
				\IF{$\scc_t = 1$}
				\STATE $\lb_t \gets \lb_{t-1}$
				\ELSE
				\STATE $\lb_t \gets \alloc_t$
				\ENDIF
				\ENDFOR
			\end{algorithmic}
			\caption{Single Arm Algorithm}\label{alg:one-arm}
	\end{center}
\end{minipage}
\end{figure}

The algorithm suffers a regret of $4(\log n +1)$:

\begin{theorem} \label{thm:ub-one-gen}
	Assume the alg. of  \algref{alg:one-arm} is invoked for $n$ iterations,
	and interacts with some arm with parameter $0 \le \amt \le 1$.
	Then \[\E R^{(n)} \le 4(\log n + 1)\]
\end{theorem}

The proof appears in \appref{sec:proof:thm:ub-one-gen}.
It consists of two parts:
first, we show that the algorithm does not waste many resources
compared to allocating $\amt$ on every iteration:
\[
	\sum_{t=1}^n (\alloc_t - \amt)
	= \sum_{t=1}^n \left(\lb_{t-1} + \frac{2}{t}- \amt \right)
	\le \sum_{t=1}^n \frac{2}{t}
	\le 2(\log + 1).
\]
Secondly, we bound the expected error $\E[\amt - \lb_t]$ of the lower
bound estimate on iteration $t$, using the simple recursive inequality:
$\E[\amt - \lb_t] \le \E[\amt - \lb_{t-1}] \left( 1 - \frac{2}{t} \right) + \frac{4}{t^2}$.
One obtains that $\E[\amt - \lb_t] = O(1/t)$, which, in tern,
implies a low number of failures:
$\E\left[ \sum_{t=1}^n \left( 1 - \scc_t \right) \right] \le 2(\log n + 1)$.
A bound on the regret is obtained by combining these two bounds.
The proof holds for a more general and adversarial setting,
as discussed in \remref{rem:gen:sin}.

\begin{remark}\label{rem:gen:sin}
  The algorithm of \figref{alg:one-arm} and the analysis in
  \thmref{thm:ub-one-gen} hold for the
  following gemeral setting where the success probability of the arm
  has two restrictions: (1) if $\alloc_t \ge \amt$, then
  $\scc_t = 1$ with probability 1, and (2) for any values of $t$,
  $\alloc_1, \dots, \alloc_t$ and $\scc_1, \dots, \scc_{t-1}$ for
  which $\alloc_t < \amt$, we have, 
		\(
			\Pr \left[ \scc_t = 0 \mid \alloc_1 \cdots \alloc_t \scc_1 \cdots \scc_{t-1} \right]
			\ge \amt - \alloc_t.
		\)
The second restriction ensures that the optimal allocation is always $\alloc_t = \amt$.

\end{remark}

\section{Multi-Arm Problem}
We address the following problem presented by \cite{LCS}, as we describe
briefly.  There are $\nban$ arms denoted by $1, 2, \dots, \nban$.
On each iteration $t$ an algorithm divides a resource between the arms,
such that arm $k$ receives $\alloc_{k,t} \geq 0$ of it. We assume that
the total amount of resource is bounded,
$\sum_k \alloc_{k,t} \leq 1$. The success probability of each arm
given $\alloc_{k,t}$ is $\min(1, \alloc_{k,t}/\amt_k)$,
where $\amt_k$ is a fixed unknown parameter associated with arm
$k$. If the amount allocated $\alloc_{k,t}$ is greater than this threshold
$\amt_k$, then the arm will succeed with probability
$1$. Otherwise, it will succeed with probability proportional to the
amount allocated: $\alloc_{k,t}/\amt_k$. Finally, define
$\amt = (\amt_1, \dots, \amt_\nban)$, and assume that
$\amt_1 \le \cdots \le \amt_\nban$ (the algorithm does not know
this ordering).
%
%
%
Denote the success indicator by $\scc_{k,t}$ and set
$\scc_{k,t}= 1$ if the arm succeeds and $0$ if it fails.
The goal of the algorithm is to maximize the number of success pulls
after $n$ iterations, called the {\em reward} and given by
\(
\sum_{t=1}^n \sum_{k=1}^\nban \scc_{k,t}. 
\)

\cite{LCS} described an algorithm to find the optimal allocation when the thresholds $\amt_1, \dots, \amt_k$ are known. 
This allocation is obtained by prioritizing the arms according to the amount of resource they require ($\amt_k$).
First, the arm with the lowest requirement is allocated with the minimal amount required to succeed with probability $1$, that is $\alloc^*_1 = \amt_1$, then the second lowest, and so on, until either there is no resource left, or all arms receive the amount they require.
Formally, 
this optimal allocation is defined recursively, and arm $k$ is allocated with, 
\(
\alloc^*_k = \min\left(\amt_k, 1 - \sum_{i=1}^{k-1} \alloc^*_i\right).
\)
Let $\nal$ be the number of arms $k$ for which $\alloc^*_k = \amt_k$.
It holds that for all $1 \le k \le \nal$, $\alloc^*_i = \amt_i$.
If $\nal < \nban$ then $0 < \alloc^*_{\nal+1} < \amt_{\ell+1}$ and define $\rem = \alloc^*_{\nal + 1}$.
The expected reward from this optimal allocation is $\E \left[ \sum_k X_k \right] = 
\nal + \indi_{\nal < \nban} \frac{\rem}{\amt_{\nal+1}}$,
where $\indi_A$ denotes an indicator for $A$.

Assuming the executed algorithms do not know the parameters of $\amt_1, \dots,
\amt_\nban$ neither their ordering, they are expected to obtain less
reward than the optimal allocation. We call the difference between an 
algorithm's \emph{actual} reward and the optimal \emph{expected}
reward (over all randomizations) by {\em regret} given by, 
%
\[
\regret^{(n)}(A, \amt) = \sum_{t=1}^n
  \regret_t = n \left(\nal + \indi_{\nal < \nban} \frac{\rem}{\amt_{\nal+1}}\right) -
 \sum_{t=1}^n \sum_{k=1}^\nban \scc_{k,t}~,
\]
where $A$ denotes the algorithm.
The goal of any algorithm is to minimize the expected regret.

\paragraph{Our Contribution: }

We describe in \algref{alg:multi-arm} an algorithm that receives a
parameter $\mulc > 2$ as input, and operates in the above setting,
with a regret $O(\log \nit)$, and constants depending on the threshold
parameters $\amt$. This improves over the previous bound of $O(\log^2 \nit)$
of  \cite{LCS}.  We also present a lower bound showing that the
dependence in $n$ cannot be improved.
It is impossible to get a polylogarithmic regret independently on the
problem parameters as shown by \cite{LCS}. 


Besides having a lower regret bound compared to the algorithm of
\cite{LCS}, our algorithm does not have to know the value $n$ in advance
(without having to rely on a doubling trick), and has a lower
initialization cost.  Also, whenever $\nban \le \nal + 1$, our algorithm
shows a great superiority in the simulations, and it performs
considerably better in general.
%
In the next
theorem we state an upper bound on the regret of the presented algorithm (\algref{alg:multi-arm}).
\begin{theorem} \label{thm:multi-arm}
	Fix some $\mulc > 2$, and let $A_\mulc$ denote the algorithm of \figref{alg:multi-arm} invoked with the parameter $c$. 
	Fix an integer $\nban > 0$, and a vector $\mathbf{\amt} \in \mathbb{R}_+^\nban$ and an integer $n > 1$.
	Then, 
	\[
		\E R^{(n)}(A_c, \mathbf{\amt}) \le C \nal \log n + C_1 \log n + C_2,
	\]
	where $C > 0$ is a constant that depends only on $\mulc$, and 
	\begin{align*}
		C_1 &= C \cdot \left(\frac{\amt_{\nal+1}}{\amt_{\nal+1} - \amt_\nal} + \sum_{k = \nal + 2}^{\nban} \frac{\amt_k}{\amt_k - \amt_{\nal+1}} \right) \\
		C_2 &= C \cdot \left( (\nal + 1) \max(1, \log \frac{1}{\amt_1}) + \nban \log \nban \right) ~.
	\end{align*}
\end{theorem}
The bound has better dependence in $n$ and the constants are compared with
the bound of \cite{LCS} with regret of the form, 
\(
\nal \log^2 n + \log n \sum_{k = \nal + 2}^{\nban} \frac{\amt_k}{\amt_k - \amt_{\nal+1}},
\) 
plus some terms independent on $n$.

Next, we present a lower bound of $\Omega(\ell n)$ on the regret. The proof appears in \secref{sec:pr:lb}, and a different lower bound is presented and proved in \secref{sec:lb-nu}.
\begin{theorem} \label{thm:lb-main} Fix an integer $r > 0$ and define
  $\nban = 2r$.  Let $\mathcal{D}$ be the following
 probability space
  over vectors $\mathbf{\amt} \in \mathbb{R}^\nban$: $\amt_1, \dots, \amt_r$
  are picked uniformly and independently from
  $\left[ \frac{1}{2r}, \frac{1}{r} \right]$, and
  $\amt_{r+1} = \cdots = \amt_{2r} = \frac{2}{r}$.  Then, for
  $\amt \sim \mathcal{D}$ and $H(n) = \sum_{i=1}^n \frac{1}{i}$, any
  algorithm $A$ satisfies, 
	\[
	\E \regret^{(n)}(A, \mathbf{\amt}) \ge \frac{r}{32} \left( H(n-1) - \frac{\pi^2}{12} \right).
	\]
\end{theorem}
Here is an intuition for the proof. For any $t \ge 1$, the total variation distance between the first $t$ successes ($X_{k,1} \cdots X_{k,t}$) of an arm with paramter $\amt$ and the successes of an arm with parameter $\amt'$ is at most $O(t \lvert 1/\amt - 1/\amt' \rvert)$. Hence, $t = \Omega(1/\lvert 1/\amt - /1\amt' \rvert)$ rounds are required to distinguish between $\amt$ and $\amt'$. This roughly implies that under the distribution $\mathcal{D}$ in \thmref{thm:lb-main}, one can estimate $\amt_1, \dots, \amt_r$ with an additive error not lower than $\Omega(1/(rt))$, hence the regret incurred at round $t$ by misallocating any arm $k$ is $\Omega(1/t)$. Summing over arms $1 \le k \le r$ and over all rounds $1 \le t \le n$, one obtains a regret of $\Omega(r \log n)$.

%

\section{Algorithm} \label{sec:alg}

In this section, we present the algorithm and an intuition to its construction.
Recall the optimal allocation algorithm which knows the parameters $\amt_1\cdots\amt_\nban$ and allocates resource to the arms in an escending order of $\amt_k$: arms $1$ to $\ell$ are fully allocated, arm $\ell+1$ receives the remaining resource and the rest of the arms receive no resource (ussuming wlog that $\amt_1 \le \cdots \le \amt_\nban$). The algorithm of \cite{LCS} uses the same algorithm, replacing the real parameter $\amt_k$ by a lower bound estimate $\lb_{k,t-1}$ obtained on iteration $t$: the arms receive resource in an escending order of the lower bound estimate, each arm $k$ receiving $\lb_{k,t-1}$ resource, until no resource is left. Their estimates $\lb_{k,t-1}$ converges to $\amt_k$, which implies that the allocations in their algorithm converge to the optimal allocation.

One would suggest using the scheme of \cite{LCS} while replacing their lower bound estimate with the one suggested in \secref{sec:single}, however, there are some obstacles which enforce the solution to be more involved. Recall that in \secref{sec:single} the arm was allocated with $\alloc_t = \lb_{t-1} + c/t$ resources where $c=2$ (in the multi armed algorithm we allow $c$ to be any constant greater than $2$). Since there are multiple arms, this solution would be wasteful: one would possibly allocate a redundant amount of $c/t$ per arm. Similarly to \thmref{thm:ub-one-gen}, one can show that an allocation of $\alloc_{k,t} = \lb_{t-1} + c\amt_k/t$ is sufficient. Since $\amt_k$ is unknown, it is replaced with its lower bound estimate, denoted $\lbf_{k,t-1}$.

Here is another issue: one cannot allocate $\lbf_{k,t-1} + \amt_k c/t$ resources on any iteration due to two reasons. First, one replaces $\amt_k$ with $\lbf_{k,t-1}$, a bound which may be inaccurate, at least on the beginning.  Secondly, due to a lack of resources, it may happen that one, for instance, would allocate an amount higher than $\lbf_{k,t-1}$ and lower than $\lbf_{k,t-1} + \lbf_{k,t-1} c/t$. Due to this issue, the solution of allocating $\lbf_{k,t-1}+\lbf_{k,t-1}c/t$ would not work. The value $\amt_k c/t$ is replaced with an amount which depends on all previous allocations: one sets $\saf_{k,t} = \sum_{i \le t} \max\{ 0, \alloc_{k,i} - \lbf_{k,i-1} \}$\footnote{This sum does not include the initialization rounds defined below.} and $\add_{k,t} = c \lbf_{k,t-1} \exp\left( -\saf_{k,t-1}/\left(c\lbf_{k,t-1} \right)\right)$, and allocates $\alloc_{k,t} = \lbf_{k,t-1} + \add_{k,t}$ if there are sufficients resources. This definition makes sense: the sequence $a_1, a_2, \dots$ defined by $a_1 = c\amt_k$ and $a_{t} = c \amt_k \exp\left(-\sum_{i=1}^{t-1} a_i / \left(c \amt_k\right)\right)$ satisfies $a_t \approx c \amt_k / t$. Hence, have the two issues described in the beginning of this paragraph not existed, the new allocation scheme would have allocated an amount similar to $\lbf_{k,t-1} + c \amt_k /t$.

An algorithm based only on $\lbf_{k,t}$ would not achieve the desired regret. A tipical situation is that the algorithm allocates any arm $k \in \{1,\dots, \ell\}$ with an amount similar to $\amt_k$, and only a small amount of resource remains for the next arm, an amount insufficient for improving the estimate: one can improve $\lbf_{k,t}$ over $\lbf_{k,t-1}$ only when $\alloc_{k,t} > \lbf_{k,t-1}$. Without being able to improve the estimates on the remaining arms, one cannot accurately decide which arm should get the remaining resource. For that reason, we create another estimate, inspired by the estimate of \cite{LCS} and by the UCB algorithm of \cite{UCB}. It is denoted by $\lbs_{k,t}$, as it is \textbf{p}robabilistic, while $\lbf_{k,t}$ is a \textbf{d}eterministic bound. This bound relies on the fact that $\mathbb{E}\left[ \scc_{k,t}\mid \alloc_{k,t} \right] = \alloc_{k,t}/\amt_k$ whenever $\alloc_{k,t} \le \amt_k$. It estimates $1/\lbs_{k,t} \approx \left(\sum_i \scc_{k,i} \right) / \left( \sum_i \alloc_{k,i} \right)$ where the sum is over all $i \le t$ such that $\alloc_{k,i} \le \lbf_{k,i-1}$: for these values of $i$ it is guaranteed that $\alloc_{k,i} \le \amt_k$. The actual estimate is slightly lower as one requires that $\lbs_{k,t} \le \amt_k$ with high probability. See \figref{alg:multi-arm} a full definition of $\lbs_{k,t}$. The resource is allocated to the arms in an ascending order of $\max\left( \lbs_{k,t}, \lbf_{k,t} \right)$.

One gets into the following dilema: what happens if, at some point, the remaining resources is higher than $\lbf_{k,t-1}$ and lower than $\lbf_{k,t-1}+\add_{k,t}$, where $k$ is the next arm to be allocated. Here are two unsuccessful solutions:
\begin{itemize}
	\item
	Allocating all the remaining resources to arm $k$: as a result, the estimate $\lbf_{k,t}$ may improve over $\lbf_{k,t-1}$, however, not as good as the improvement when allocating $\alloc_{k,t}=\lbf_{k,t-1}+\add_{k,t}$. 
	Additionally, the estimate $\lbs_{k,t}$ cannot improve after allocating more than $\lbf_{k,t-1}$ resource, hence it does not improve.
	This slow improvement of $\max(\lbf_{k,t},\lbs_{k,t})$ could imply that the arm will get a priority it does not deserve for many rounds, taking resources which could better be utilized by other arms.
	\item
	Allocating $\lbf_{k,t-1}$ resources: as a result, the estimate $\lbs_{k,t}$ will improve over $\lbs_{k,t-1}$, however $\lbf_{k,t}$ will not. Since only $\lbf_{k,t-1}$ resources are allocated rather than all remaining resources, arm $k$ may get stuck, receiving the same amount of resources on every iteration, while the remaining resources are given to inferior arms.
\end{itemize}
One can solve this problem by making sure that both $\lbs_{k,t}$ and $\lbf_{k,t}$ are improved with constant probability, tossing an unbiased coin to decide between allocating all the remaining resources to arm $k$ and allocating $\lbf_{k,t-1}$ resources.

Due to the definition of $\add_{k,t}$, our allocation scheme requires $\lbf_{k,t}$ to be positive. In order to obtain an initial positive estimate $\lbf_{k,t}$, a different allocation scheme is performed, similarly to the initialization phase of \cite{LCS}: each arm $k$ is allocated with $2^{-(t-1)} / \nban$ resources on every iteration $t$ until it fails ($\scc_{k,t}=0$). Then, $\lbf_{k,t}$ is set as the amount $\alloc_{k,t}$ allocated at failure, and the normal allocation scheme is used from then.

The algorithm appears in \figref{alg:multi-arm}. As one may notice, it may be implemented using $O(\nban)$ memory and $O(\nban \log \nban)$ time per iteration\footnote{The algorithm contains sums over $i=1,\dots,t$, however, one can calculate this sum given the sum up to $t-1$}. The authors did not find a simple way to implement such an efficient algorithm using existing tools. For instance, one may suggest discretizing the space of all possible allocations, and learning an allocation from this space using a standard multi armed bandit \citep{UCB}. However, in order to achieve a polylogarithmic regret, $\tilde{\Omega}(n)^\nban$ different arms are required, which is high even for the setting with $\nban = 1$. Another suggestion it to estimate $\amt_1, \dots, \amt_\nban$ using a maximum likelihood estimator, calculating
\begin{equation} \label{eq:mle}
\arg\max_{\amt_k > 0}\Pr\left[ \scc_{k,1} \cdots \scc_{k,t} \middle| \amt_k \alloc_{k,1} \cdots \alloc_{k,t} \right] 
= \arg\max_{\amt_k > 0}\prod_{i=1}^t \left(1 - \scc_{k,t}-\min \left\{ \alloc_{k,t} / \amt_k, 1 \right\}\right)
\end{equation}
for any arm $k$. However, it seems that any simple implementation requires that $\alloc_{k,t}/\amt_k \le 1$, a solution offered by \cite{LCS}\footnote{\cite{LCS} used confidence intervals instead of a maximum likelihood estimator.} which suffers a higher regret. Otherwise, the authors think that there is no simple way to calculate this estimate for all $t$ without storing $\alloc_{k,i}$ and $\scc_{k,i}$ in memory for all $i \le t$.

\begin{figure}
	\begin{center}
		
		\begin{algorithmic}[1] 
			\STATE Get as an input a parameter $c > 2$
			\STATE Set $\lbf_{k,0} \gets 0$ and $\lbs_{k,0}\gets 0$ for all $k \in \{1, \dots, \nban\}$.
			\FORALL {$t \gets 1, 2, \dots$}
			\STATE resource $\gets 1$
			\FORALL {$k \in \{1,\cdots, \nban \}$ in an increasing order of $\max(\lbf_{k,t-1},\lbs_{k,t-1})$} \label{algl:k-loop}
			\STATE $\add_{k,t} \gets \mulc \lbf_{k,t-1} \exp\left(- \frac{\saf_{k,t-1}}{\lbf_{k,t-1} \mulc} \right)$ if $\lbf_{k,t-1}> 0$ otherwise $\add_{k,t} \gets 0$
			\IF {$\lbf_{k,t-1} = 0$}
			\STATE $\alloc_{k,t} \gets \frac{1}{\nban 2^{t-1}}$ \COMMENT{Case I}
			\ELSIF{$\text{resource} \ge \lbf_{k,t-1} + \add_{k,t}$}
			\STATE $\alloc_{k,t} \gets \lbf_{k,t-1} + \add_{k,t}$ \COMMENT{Case A}
			\ELSIF{$\lbf_{k,t-1} < \text{resource} < \lbf_{k,t-1} + \add_{k,t}$}
			\STATE Draw an unbiased coin to decide whether $\alloc_{k,t} \gets \lbf_{k,t-1}$ or $\alloc_{k,t}  \gets \text{resource}$ \COMMENT{Case B}
			\ELSE
			\STATE $\alloc_{k,t} \gets \text{resource}$ \COMMENT{Case C} 
			\ENDIF
			\STATE resource $\gets \text{resource} - \alloc_{k,t}$
			\ENDFOR
			\STATE Observe $\scc_{1,t}, \dots, \scc_{\nban,t}$
			\STATE $\lbf_{k,t}\gets \max_{i \le t \colon \scc_{k,i} = 0} \alloc_{k,i}$ if the max is over a nonempty set, otherwise $\lbf_{k,t} \gets 0$
			\STATE $\saf_{k,t} = \sum_{i\le t \colon \lbf_{k,i-1}>0} \max\left\{ \alloc_{k,i} - \lbf_{k,i-1}, 0 \right\}$ 
			\STATE $\epsilon_{t} = t^{-3} \nban^{-1}$\quad;\quad$\zeta_t \gets \left( \sqrt{1/2} + \sqrt{1/2 - \log \epsilon_t} \right)^2$
			\STATE $\sas_{k,t}= \sum_{i\le t ~\colon~ \alloc_{k,i} \le \lbf_{k,i-1}} \alloc_{k,i}$\quad; \quad
			$\xs_{k,t} = \sum_{i\le t ~\colon~
				\alloc_{k,i} \le \lbf_{k,i-1}} \scc_{k,i}$
			\STATE $\lbs_{k,t} \gets \left( \sqrt{\frac{\zeta_{t}}{2 \sas_{k,t}}} + \sqrt{\frac{\zeta_{t}}{2 \sas_{k,t}} + \frac{\xs_{k,t}}{\sas_{k,t}}} \right)^{-2}$ if $\sas_{k,t} > 0$ otherwise $\lbs_{k,t} \gets 0$
			\ENDFOR
		\end{algorithmic}
	\end{center}
	\figline
	
	\caption{Resource-allocation algorithm for the multi-armed problem.}
	\label{alg:multi-arm}
\end{figure}
\section{Proof Outline of Theorem~\ref{thm:multi-arm}} \label{sec:pr-outline}

In this section the outline of \thmref{thm:multi-arm} is presented together with the main lemmas, where $c>2$ is the constant parameter given as an input to the algorithm.
Recall cases A, B, C and I from the algorithm in \figref{alg:multi-arm}.
We start by splitting the iterations into two types. Let $\lb_{k,t} = \max\left( \lbs_{k,t}, \lbf_{k,t} \right)$ and let $T$ be the set of ``good iterations'', for which $0 < \lb_{k,t-1} \le \amt_k$ for all $k$. The
core of the proof relates to iterations $t \in T$, while the number of
iterations $t \notin T$ can be bounded: first, by observing case I of the algorithm, one can show that after a short number of iterations, for all $1 \le k \le \nban$, $\lbf_{k,t} >0$. Secondly, it always holds that $\lbf_{k,t} \le \amt_k$. Lastly, the estimate $\lbs_{k,t}$ is constructed such that $\lbs_{k,t} \le \amt_k$ with high probability.
\begin{lemma} \label{lem:notin-T}
	The expected number of iterations $t \notin T$ is bounded by
	$C \max\left(\log \frac{1}{\amt_1}, 1\right)$,
	for some constant $C > 0$, depending only on $\mulc$.
\end{lemma}

From now focus on iterations $t \in T$. Note that on any iteration $t \in T$, no arm is allocated according to case I. Let $A'_t$ be the set of all arms processed in the loop over the arms in line~\ref{algl:k-loop} of \figref{alg:multi-arm} on iteration $t$ before encountering an arm $k$ which is not allocated according to case A. Let $k'_t$ be the first arm processed not according to case A. If the arm $k'_t$ is allocated according to case B then set $B'_t = \{k'_t\}, C'_t = \emptyset$ and if according to case C then $C'_t = \{k'_t\}, B'_t = \emptyset$. 
If $k'_t$ is undefined then $B_t' = C_t' = \emptyset$.
Define the sets $A_t, B_t$ and $C_t$ as the sets of all arms of $A'_t,B'_t$ and $C'_t$ (respectively) which are among the first $\ell+1$ arms processed on iteration $t$.
If $B_t \ne \emptyset$, define by $\add'_t$ the difference between the amount of resource left for $k'_t$ and $\lbf_{k'_t,t-1}$. Note that if $B_t \ne \emptyset$ then arm $k'_t$ is of case $B$, hence it will either be allocated with $\alloc_{k'_t,t} = \lbf_{k'_t,t-1}$ or with $\alloc_{k'_t,t} = \lbf_{k'_t,t-1} + \add'_t$, each with probability $1/2$. The sets $A_t$, $B_t$ and $C_t$ are defined this way only for iterations $t \in T$, and they are defined at emptysets for $t \notin T$.

Define by $Z_t$ the random variable which contains all the history up to the point where all
$\alloc_{1,t+1}, \dots, \alloc_{\nban,t+1}$ are defined and just before observing $\scc_{1,t+1} \cdots \scc_{\nban,t+1}$
(it contains the values $\{ \scc_{k,i} \}_{1\le k \le \nban,i\le t}$ and the random coins tossed in case B of the algorithm up to and including iteration $t+1$).
The expected regret on iteration $t$ given $Z_{t-1}$ equals
\[
	\E[R_t \mid Z_{t-1}] 
	= \nal + \indi_{\nban > \nal} \frac{\rem}{\amt_{\nal+1}} 
	- \sum_{k=1}^\nban \E \left[ \scc_{k,t} \mid Z_{t-1} \right]
	= \nal + \indi_{\nban > \nal} \frac{\rem}{\amt_{\nal+1}} 
	- \sum_{k=1}^\nban \min\left( 1,\ \frac{\alloc_{k,t}}{\amt_k} \right).
\]
The next lemma bounds $\E[R_t \mid Z_{t-1}]$, and decomposes it
in terms of $A_t$, $B_t$ and $C_t$.
\begin{lemma} \label{lem:general}	
	Let $t \in T$. 
	It holds that
	\begin{align}
	\E[\regret_t \mid Z_{t-1}] \le 
	&~~~~\sum_{k \in A_t} \left( 1 - \min\left( \frac{M_{k,t}}{\amt_k},1 \right)\right) \label{eq:gen-l1}\\
	&+ \sum_{k \in A_t} \frac{\add_{k,t}}{\lbf_{k,t-1}} 
	+ \sum_{k \in B_t} \frac{\add_t'}{\lbf_{k,t-1}} \label{eq:gen-l2} \\
	&+ \sum_{k \in B_t \cup C_t}
	\begin{cases}
	\min(\lbf_{k,t-1}, \alloc_{k,t}) (1/\nu_{\nal+1}-1/\nu_k) & |A_t| = \nal \\
	\min(\lbf_{k,t-1}, \alloc_{k,t}) (1/\nu_{\nal}-1/\nu_k) & |A_t| < \nal
	\end{cases}. \label{eq:gen-l3}
	\end{align}
\end{lemma}
The proof of \lemref{lem:general} matches between the allocations by the optimal allocation, and those by the algorithm.
The amount in line~\eqref{eq:gen-l1} relates to the difference between the reward of arms $1, \dots, |A_t|$ in the optimal allocation, and the reward of the members of $A_t$ in the algorithm.
The amount in line~\eqref{eq:gen-l2} relates to possibly allocating $\sum_{k \in A_t} \add_{k,t} + \add'_t$
resource to the wrong arms.
Line~\eqref{eq:gen-l3} stands for the regret incurred from allocating $\min(\lbf_{k'_t,t-1}, \alloc_{k'_t,t})$
resources to arms in $B_t \cup C_t$, instead of allocating it either to arm $\nal$ or to arm $\nal + 1$.
One can bound the total regret of the algorithm by summing the bound obtained in \lemref{lem:general} over $t \in T$ and changing the order of summation:
\begin{align}
\E R^{(n)}
&= \sum_{t=1}^n \E R_t
= \sum_{t \in T} \E R_t + \sum_{t \notin T} \E R_t
= \sum_{t\in T} \E [\E[ R_t \mid Z_{t-1}]] + \sum_{t \notin T} \E R_t \notag\\
&\le \sum_{k=1}^\nban \sum_{t \colon k \in A_t} \E \left[ 1 - \min\left(\frac{\alloc_{k,t}}{\nu_k}, 1 \right)\right] \label{eq:reg1} \\
&\qquad + \sum_{k=1}^\nban \E \left[\sum_{t \colon k \in A_t} \frac{\add_{k,t}}{\lbf_{k,t-1}} 
+ \sum_{t \colon k \in B_t}^n  \frac{\add_t'}{\lbf_{k,t-1}} \right] \label{eq:reg2}\\
&\qquad + \sum_{k=1}^\nban \E\left[ \sum_{t \colon k \in B_t \cup C_t}
\begin{cases}
\min(\lbf_{k,t-1}, \alloc_{k,t}) (1/\amt_{\nal+1}-1/\amt_{k}) & |A_t| = \nal \\
\min(\lbf_{k,t-1}, \alloc_{k,t}) (1/\amt_{\nal}-1/\amt_{k}) & |A_t| < \nal
\end{cases} \right] \label{eq:reg3} \\
&\qquad + (n-\mathbb{E}\left\vert T  \right\vert) (\nal + 1), \label{eq:reg4}
\end{align}
where the term in line \eqref{eq:reg4} is obtained from $\sum_{t \notin T} \mathbb{E}R_t$ by the fact that the reward of the optimal allocation is at most $\ell+1$, hence the regret on any iteration is at most $\ell+1$.
The regret is decomposed into four parts,
appearing in lines \eqref{eq:reg1}, \eqref{eq:reg2}, \eqref{eq:reg3} and \eqref{eq:reg4},
each bounded separately, where the amount in line \eqref{eq:reg4} is bounded by \lemref{lem:notin-T}. 

First, we bound the amount in line \eqref{eq:reg1}.
\begin{lemma} \label{lem:At} There exists a constant 
	$C > 0$,
	depending only on $\mulc$, such that for every arm 
	$k$:
	\[
	\E \left[ \sum_{t \colon k \in A_t} \left(1 -
	\min\left(1, \alloc_{k,t} / \amt_k\right) \right) \right]
	\le C(\log n + \log \nban).
	\]
\end{lemma}
To give an intuition, recall that whenever $k \in A_t$, there is a sufficient amount of resource for arm $k$, and one allocates $\alloc_{k,t} = \lbf_{k,t-1} + \add_{k,t}$. 
Note that whenever $k \in A_t$, 
\[1 - \min\left(\frac{\alloc_{k,t}}{\amt_k},1\right) = \max\left( \frac{\amt_k - \alloc_{k,t}}{\amt_k},0 \right)
\le \frac{\amt_k - \lbf_{k,t-1}}{\amt_k}. \]
Similarly to the corresponding claim in the single armed problem, one can roughly show, by a potential function calculation, that after $m$ iterations when $k \in A_t$, it holds that $\mathbb{E} \left( \amt_k - \lbf_{k,t-1}\right)/\amt_k = O(1/m)$. Hence, one can roughly bound the amount in line \eqref{eq:reg1} corresponding to any arm $k$ by $\sum_{m=1}^n O(1/m) = O(\log n)$. The actual proof is inductively by a potential function.

Next, we bound the amount in line \eqref{eq:reg2}, which corresponds to the redundant resource given to 
the arms.
\begin{lemma} \label{lem:add-bound}
	There exists some constant $C > 0$, depending only on $\mulc$,
	such that for every arm $k$: 
	\[
	\E\left[ \sum_{t \colon k \in A_t} \frac{\add_{k,t}}{\lbf_{k,t-1}} 
	+ \sum_{t \colon k \in B_t} \frac{r'_t}{\lbf_{k,t-1}} \right]
	\le C(\log n + \log \nban).
	\]
\end{lemma}
We give an intuition for the proof.
Note that if $k \in A_t$ then $\max\left(0, \alloc_{k,t}- \lbf_{k,t-1}\right)=\add_{k,t}$, and if $k \in B_t$ then $k$ is of case B, hence $\max\left(0, \alloc_{k,t}- \lbf_{k,t-1}\right)=\add'_t$ with probability $1/2$ and $\max\left(0, \alloc_{k,t}- \lbf_{k,t-1}\right)=0$ with probability $0$. Therefore, one can bound
\begin{equation} \label{eq:add-outline}
\E\left[ \sum_{t \colon k \in A_t} \add_{k,t} + \sum_{t \colon k \in B_t} \add'_t \right]
\le 2\E\left[ \sum_{t =1}^n \max\left(0, \alloc_{k,t}- \lbf_{k,t-1}\right) \right].
\end{equation}
Note that by the definition of the algorithm,
\begin{align}
\max\left(0, \alloc_{k,t}- \lbf_{k,t-1}\right) 
\le \add_{k,t}
&= \lbf_{k,t-1} c \exp\left( - \frac{\sum_{i =1}^{t-1} \max\left(0, \alloc_{k,i}- \lbf_{k,i-1}\right)}{c \lbf_{k,i-1}} \right)\label{eq:wrong-sum}\\
&\le c \amt_k \exp\left( - \frac{\sum_{i =1}^{t-1} \max\left(0, \alloc_{k,i}- \lbf_{k,i-1}\right)}{c \amt_k} \right), \notag
\end{align}
where the last inequality follows from the fact that $\lbf_{k,t-1}\le \amt_k$ and the fact that $x e^{-\alpha/x}$ is monotonic nondecreasing in $x$ for $\alpha \ge 0$ and $x>0$\footnote{Note the sum in the right hand side of line \eqref{eq:wrong-sum} is over all $i \le t-1$. While the definition of $\add_{k,t}$ requires the sum to be over all $i \le t-1$ such that $\lbf_{k,i-1}>0$, we ignore this requirement, for simplicity of presentation.}.
One can show that this implies that
$\sum_{t =1}^n \max\left(0, \alloc_{k,t}- \lbf_{k,t-1}\right) \le \sum_{t=1}^n a_t$, where $a_1 = c \amt_k$ and $a_t = c \amt_k \exp\left(-\sum_{i=1}^{t-1} a_i / (c \amt_k)\right)$ for all $t>1$. It holds that $a_t \approx c\amt_k/t$, which implies that $\sum_{t=1}^n a_t = O(\amt_k\log n)$. Combining the last inequalities, one obtains a bound of $\amt_k \log n$ on the left hand side of Eq. \eqref{eq:add-outline}. This concludes the proof since $\lbf_{k,t-1} = \Omega(\amt_k)$ for most values of $t$.

Lastly, we bound on the amount in line \eqref{eq:reg3}, inspired by \cite{LCS} and \cite{UCB}.
\begin{lemma} \label{lem:slow}
	There exists some constant $C > 0$, depending only on $\mulc$, such that for every arm $k$:
	\[
	\E\left[ \sum_{t \in T \colon k \in B_t \cup C_t}
	\begin{cases}
	\min(\lbf_{k,t-1}, \alloc_{k,t}) (1/\nu_{\nal+1}-1/\nu_{k}) & |A_t| = \nal \\
	\min(\lbf_{k,t-1}, \alloc_{k,t}) (1/\nu_{\nal}-1/\nu_{k}) & |A_t| < \nal
	\end{cases} \right]
	\le \begin{cases}
	C \frac{\amt_k}{\amt_k - \amt_{\nal + 1}} \log n & k > \nal + 1 \\
	C \frac{\amt_k}{\amt_k - \amt_{\nal}} \log n & k = \nal + 1 \\
	0 & k < \nal + 1
	\end{cases}.
	\]
\end{lemma}
We give an intuition for the proof, ignoring the dependency on $\amt_1 \cdots \amt_\nban$ for simplicity. Recall that $1/\lbs_{k,t-1}$ is estimated roughly by the number of successes divided by the total resource, $\left(\sum_i \scc_{k,i}\right) /\sum_i \alloc_{k,i}$ over iterations $i \le t-1$ for which $\alloc_{k,i} \le \lbf_{k,i-1}$. For a single $i$ in the sum, expectation of $\scc_{k,i}/\alloc_{k,i}$ is indeed $1/\amt_k$, and
a relative Chernoff bound can show that if $\sum_i \alloc_{k,i}$ is sufficiently large then this estimate is close to $1/\amt_k$ with high probability. Fix some $k > \ell+1$ and if $\sum_i \alloc_{k,i} = \Omega(\log n)$ for a sufficiently large constant, then $\lbs_{k,t-1} > \amt_{\ell+1}$. If $t \in T$ this implies that $\lbs_{k,t-1} >\amt_{\ell+1} \ge \lb_{1,t-1}, \dots, \lb_{\ell+1,t-1}$ and $k$ is not one of the first $\ell+1$ arms processed on iteration $t$. Hence, $k$ is not in $B_t \cup C_t$ from that point onwards, which implies that $\mathbb{E} \sum_t \alloc_{k,t} = O(\log n)$, where the sum is over iterations $1\le t \le n$ such that $\alloc_{k,t} \le \lbf_{k,t-1}$ and $k \in B_t \cup C_t$. Since $B_t$ and $C_t$ contain arms of cases B and C respectively, whenever $k \in C_t$ it holds that $\alloc_{k,t} \le \lbf_{k,t-1}$ and whenever $k \in B_t$ then $\alloc_{k,t} \le \lbf_{k,t-1}$ with probability $1/2$. In particular, this implies that 
\[
\mathbb{E} \left[\sum_{t \colon k \in B_t \cup C_t} \min(\lbf_{k,t-1}, \alloc_{k,t}) \right]
\le 2\mathbb{E} \left[\sum_{\substack{t \colon k \in B_t \cup C_t\\\alloc_{k,t}\le \lbf_{k,t-1}}} \min(\lbf_{k,t-1}, \alloc_{k,t}) \right]
= 2\mathbb{E} \left[\sum_{\substack{t \colon k \in B_t \cup C_t\\\alloc_{k,t}\le \lbf_{k,t-1}}} \alloc_{k,t}\right].
\]
The last term is $O(\log n)$, which concludes the lemma for any arm $k > \ell+1$. One can similarly bound the amount corresponding to $k = \ell+1$, while the amount corresponding to $k \le \ell$ is non-positive since $1/\amt_{\ell}- 1/\amt_k \le 0$.

\section{Simulations}
\begin{figure}
\centering
\includegraphics[width=0.51\textwidth]{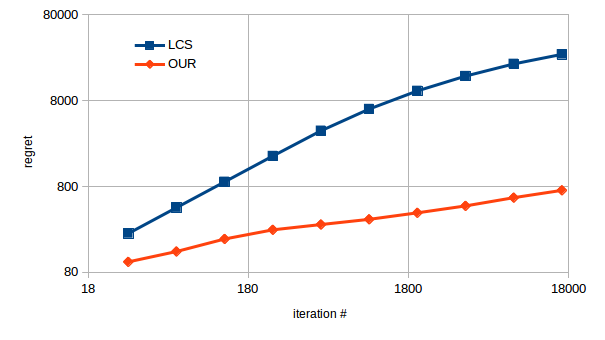}
\caption{\small{ Regret vs n for two algorithms (log scale).}}
\label{fig:sims}
\end{figure}
We conducted 
simulations to evaluate the merits of our
methods, each for $100$ executions. First, we followed the choice of \cite{LCS} and used a
problem with $\nban = 2$ and $\amt_1 = 0.4, \amt_2 = 0.6$ as a problem
where the regret contains only a term of the form $\nal \log^2 n$, and
indeed found out that the regret behaves as $45 \log^2 n$. We remind
the reader that the main improvement of our algorithm is by replacing 
the $\nal \log^2 n$ term with $\nal \log n$.  This term corresponds to the
regret obtained from the fact the algorithm does not know the exact
requirements ($\amt_k$) of the top $\nal$ arms. 
We experimented with 
$\log_2 n = 1,2, \dots, 18$, and $\mulc = 2.5$, and the regret behaves
as $3.5 \log n$ with high confidence. 
For $n = 2^{18}$ this is an improvement from $7053$ to $43$.

While our main improvement in the regret corresponds to reducing the term $\ell \log^2 n$ to $\ell \log n$, the other main term, $\log n \sum_{k=\ell+2}^\nban \frac{\amt_k}{\amt_k - \amt_{\ell-1}}$, which corresponds to arms $k > \ell+1$, appears in both papers. Hence, one expects that the greatest difference between the algorithms would be in situations where $\nban/\ell$ is low. Indeed, this is the case, as shown in our simulations.

We also performed experiments where the arm parameters $\amt_k$ are
uniformly spanned.  One execution was performed with $\nban \!\!=\!\! 50$, and
$\amt_k \!\!=\!\! \frac{2k}{25^2}$ for $k \!\!=\!\! 1, \dots, 50$. That is,
$\amt_1 \!\!=\!\! 2/25^2$, $\nu_{50}\!\!=\!\!100/25^2\!\!=\!\!4/25$, and $\nal \!\!=\!\! 24$. The regret vs
$n$ is plotted in \figref{fig:sims}. In each of the 100 executions, we ran one copy of our algorithm
as it is any-time, yet multiple-copies of the algorithm of \cite{LCS}: one
for each value of the horizon $n$. For $n\!\!=\!\!2^{14}$ our algorithm
suffers a regret of $721$ compared to $27,681$ by their algorithm.

Similar trends were observed with other choices of the parameters. For
example, with $\nban = 100$, and $\amt_k = \frac{2k}{100^2}$ for
$k = 1, \dots, 100$. Here $\nal = 99$, therefore only the
term $\ell \log n$ takes part, and for $n = 2^{18}$ our algorithm suffers a regret
of $1167$ compared to $352,173$ by their. Another example is when we
set $\nban = 50$ and $\amt_k = \frac{2k}{10^2}$ for
$k = 1, \dots, 50$ (therefore $\nal = 9$). For a horizon of $n = 2^{18}$
our algorithm suffers a regret of $1,544$ compared to the $21,665$ by
the benchmark. The regret of our algorithm in these two experiments as function of $n$ is drawn in \figref{fig:2-sims}, where the $x$-axis is in logarithmic scale and the $y$ axis is in a normal scale. One can see that in the first experiment, the regret is a linear function of $\log n$, while in the second experiment, the regret is a linear function of $\log n$ for any value $n \ge 2^{15}$ (we executed up to $n=2^{20}$).
\begin{figure}
	\centering
	\begin{subfigure}[b]{0.5\textwidth}
		\centering
		\includegraphics[height=1.6in]{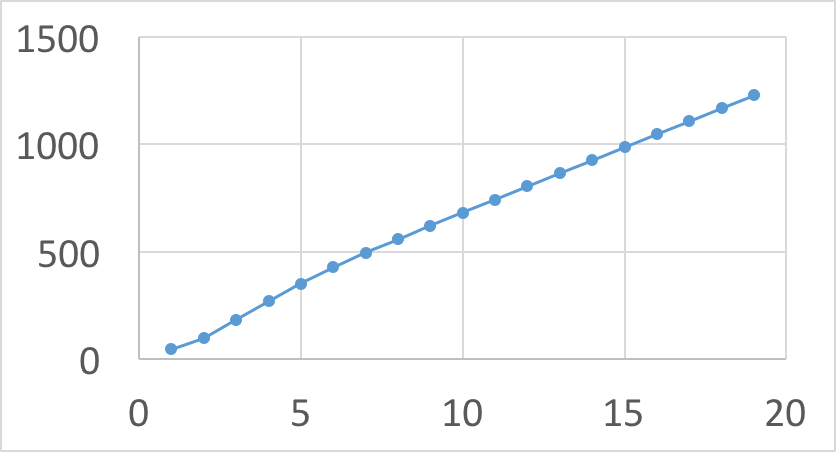}
		\caption{100 arms}
	\end{subfigure}%
	~ 
	\begin{subfigure}[b]{0.5\textwidth}
		\centering
		\includegraphics[height=1.5in]{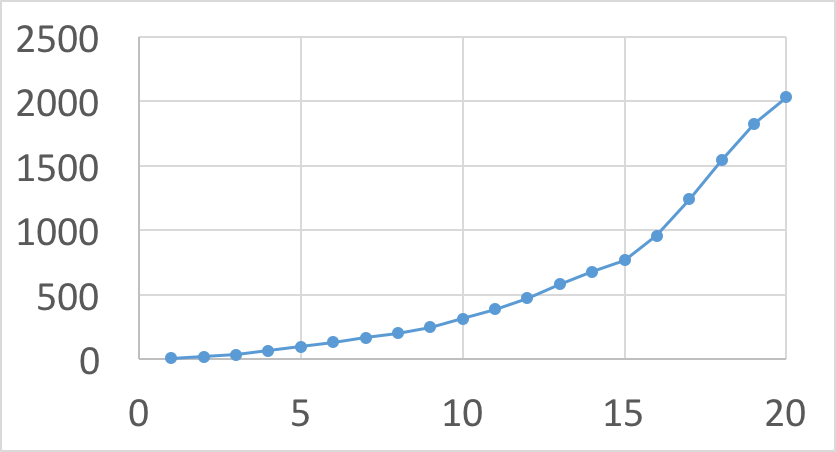}
		\caption{50 arms}
	\end{subfigure}
	\caption{Regret of the algorithm in \figref{alg:multi-arm} as a function of $\log_2 n$}
	\label{fig:2-sims}
\end{figure}
\section{Summary}
We described an algorithm for the multi-resource allocation problem
and proved both upper and lower regret bounds of $\Theta(\log n)$, an
improvement compared to the regret of $O(\log^2 n)$ of the previous
algorithm by \cite{LCS}. Additionally, we discussed a related settings, where
there is only a single-arm.
%
%
Simulations we performed showed
the supervisory of our algorithm. Future directions are extending our
results to the multi-resource problem~\citep{LCS2}, to the contextual
case where algorithms receive instance dependent side
information, and to the case where the parameters or total amount of
resource drifts in time.
Lastly, we believe that the algorithm can be modified to handle non linear bandits,
similarly to the generalization of the one arm problem in \remref{rem:gen:sin}.

\bibliography{paper}

\begin{thebibliography}{10}
\providecommand{\natexlab}[1]{#1}
\providecommand{\url}[1]{\texttt{#1}}
\expandafter\ifx\csname urlstyle\endcsname\relax
  \providecommand{\doi}[1]{doi: #1}\else
  \providecommand{\doi}{doi: \begingroup \urlstyle{rm}\Url}\fi

\bibitem[Abbasi-Yadkori et~al.(2011)Abbasi-Yadkori, P{\'a}l, and
  Szepesv{\'a}ri]{AYSC11}
Yasin Abbasi-Yadkori, D{\'a}vid P{\'a}l, and Csaba Szepesv{\'a}ri.
\newblock Improved algorithms for linear stochastic bandits.
\newblock In \emph{Advances in Neural Information Processing Systems}, pages
  2312--2320, 2011.

\bibitem[Agrawal and Goyal(2013)]{AR13}
Shipra Agrawal and Navin Goyal.
\newblock Thompson sampling for contextual bandits with linear payoffs.
\newblock In \emph{ICML (3)}, pages 127--135, 2013.

\bibitem[Auer et~al.(2002)Auer, Cesa-Bianchi, and Fischer]{UCB}
Peter Auer, Nicolo Cesa-Bianchi, and Paul Fischer.
\newblock Finite-time analysis of the multiarmed bandit problem.
\newblock \emph{Machine learning}, 47\penalty0 (2-3):\penalty0 235--256, 2002.

\bibitem[Broder and Rusmevichientong(2012)]{BM12}
Josef Broder and Paat Rusmevichientong.
\newblock Dynamic pricing under a general parametric choice model.
\newblock \emph{Operations Research}, 60\penalty0 (4):\penalty0 965--980, 2012.

\bibitem[Chernoff(1952)]{CHE52}
Herman Chernoff.
\newblock A measure of asymptotic efficiency for tests of a hypothesis based on
  the sum of observations.
\newblock \emph{The Annals of Mathematical Statistics}, pages 493--507, 1952.

\bibitem[Kveton et~al.(2015)Kveton, Wen, Ashkan, and Szepesvari]{KBW15}
Branislav Kveton, Zheng Wen, Azin Ashkan, and Csaba Szepesvari.
\newblock Tight regret bounds for stochastic combinatorial semi-bandits.
\newblock In \emph{AISTATS}, 2015.

\bibitem[Lai and Robbins(1985)]{LR}
Tze~Leung Lai and Herbert Robbins.
\newblock Asymptotically efficient adaptive allocation rules.
\newblock \emph{Advances in applied mathematics}, 6\penalty0 (1):\penalty0
  4--22, 1985.

\bibitem[Lattimore et~al.(2014)Lattimore, Crammer, and Szepesv{\'a}ri]{LCS}
Tor Lattimore, Koby Crammer, and Csaba Szepesv{\'a}ri.
\newblock Optimal resource allocation with semi-bandit feedback.
\newblock In \emph{UAI}, pages 477--486. AUAI Press, 2014.

\bibitem[Lattimore et~al.(2015)Lattimore, Crammer, and Szepesv{\'a}ri]{LCS2}
Tor Lattimore, Koby Crammer, and Csaba Szepesv{\'a}ri.
\newblock Linear multi-resource allocation with semi-bandit feedback.
\newblock In \emph{Advances in Neural Information Processing Systems}, pages
  964--972, 2015.

\bibitem[Rajeev et~al.(1989)Rajeev, Teneketzis, and Anantharam]{RAT89}
A~Rajeev, Demosthenis Teneketzis, and Venkatachalam Anantharam.
\newblock Asymptotically efficient adaptive allocation schemes for controlled
  iid processes: Finite parameter space.
\newblock \emph{IEEE Transactions on Automatic Control}, 34\penalty0 (3), 1989.

\end{thebibliography}
\appendix
\section{Proof of Theorem~\ref{thm:ub-one-gen}}
\label{sec:proof:thm:ub-one-gen}

The proof is for the general setting discribed in Remark~\ref{rem:gen:sin}

Assume that the allocation rule of the algorithm is
$\alloc_t \gets \lb_{t-1} + \frac{\mulc}{t}$ for some $\mulc > 1$ ($2$
is replaced by $\mulc$), and bound the expected regret by
$\frac{\mulc^2}{\mulc-1}(\log n + 1)$.  Fix some arm, and let $\amt$
be its resource requirement.

We divide the expected regret into two parts:
\begin{align} 
\E \regret^{(n)}
&= \E \left[\nit(1 - \amt) - \sum_{t=1}^\nit \left( \scc_t - \alloc_t \right) \right] \nonumber \\
&= \E \left[\sum_{t=1}^\nit (\alloc_t - \amt) \right] +\E \left[\sum_{t=1}^\nit (1 - \scc_t) \right]. \label{eq:one-bandit-proof-1}
\end{align}

We start by bounding the first term of \eqref{eq:one-bandit-proof-1}.
Since the lower bound $\lb_t$ is always correct, namely, $\lb_t \le \amt$,
it holds that:
\begin{equation} \label{eq:one-bandit-proof-2}
\sum_{t=1}^\nit (\alloc_t - \amt)
= \sum_{t=1}^\nit \left(\frac{\mulc}{t} + \lb_{t-1} - \amt \right) 
\le \sum_{t=1}^\nit \frac{\mulc}{t}
\le \mulc(\log \nit + 1).
\end{equation}

Next, we bound the second term of \eqref{eq:one-bandit-proof-1}.
Define $\verr_t = \amt - \lb_t$. This is a random variable, since $\lb_t$ is also a random variable.
We will start by bounding $\E[\verr_t]$.

\begin{lemma}
	For any $0 \le t \le n$,
	\[
	\E[\verr_t] \le \frac{\mulc^2}{\mulc-1} \frac{1}{t+1}.
	\]
\end{lemma}

\begin{proof}
	We start by bounding the conditional expectation $\E[\verr_t | \verr_{t-1}]$, for any $1 \le t \le \nit$.
	Fix some $1 \le t \le \nit$ and $0 \le \err \le 1$, and assume that $\verr_{t-1} = \err$.
	The problem definition assumes that the probability that $\scc_t = 0$ is at least
	\[
	\amt - \alloc_t
	= \amt - \lb_{t-1} - \frac{\mulc}{t} 
	= \verr_{t-1} - \frac{\mulc}{t}
	= \err - \frac{\mulc}{t}.
	\]
	Denote $p = \Pr[\scc = 0 | \verr_{t-1} = \err]$. As we have just showed,
	$p \ge \err - \frac{\mulc}{t}$.
	By the definition of the algorithm, with probability $p$, $\lb_t = \alloc_t$. At that case,
	\[
	\verr_t 
	= \amt - \lb_t
	= \amt - \alloc_t
	= \amt - \lb_{t-1} - \frac{\mulc}{t}
	= \verr_{t-1} - \frac{\mulc}{t}
	= \err - \frac{\mulc}{t}.
	\]
	With probability $1-p$, $\lb_t = \lb_{t-1}$. At that case,
	\[
	\verr_t
	= \amt - \lb_t
	= \amt - \lb_{t-1}
	= \verr_{t-1}
	= \err.
	\]
	Therefore,
	\[
	\E[\verr_t | \verr_{t-1} = \err]
	= p (\err - \frac{\mulc}{t}) + (1-p) \err
	= \err - p \frac{\mulc}{t}
	\le \err - (\err - \frac{\mulc}{t}) \frac{\mulc}{t}
	= \err - \err \frac{\mulc}{t} + \frac{\mulc^2}{t^2}.
	\]
	Writing it differently, this means that
	$\E[\verr_t | \verr_{t-1}] \le \verr_{t-1} (1- \frac{\mulc}{t}) + \frac{\mulc^2}{t^2}$.
	
	We conclude the lemma by induction on $0 \le t \le \nit$.
	For $t = 0$,
	\[
	\E[\verr_0]
	= \E[\amt - \lb_0]
	= \amt
	\le 1
	\le \mulc
	\le \frac{\mulc^2}{\mulc-1}
	= \frac{\mulc^2}{\mulc-1} \frac{1}{t+1}.
	\]
	For $1 \le t \le \nit$,
	\begin{align*}
	\E[\verr_t]
	&= \E[\E[\verr_t | \verr_{t-1}]] \\
	&= \E[\verr_{t-1} (1 - \frac{\mulc}{t}) + \frac{\mulc^2}{t^2}]\\
	&\le \frac{\mulc^2}{(\mulc-1)t} (1 - \frac{\mulc}{t}) +  \frac{\mulc^2}{t^2}\\
	&= \frac{\mulc^2}{(\mulc-1)t^2} (t - \mulc) +  \frac{\mulc^2}{(\mulc-1)t^2}(\mulc-1)\\
	&= \frac{\mulc^2}{(\mulc-1)}\frac{t-1}{t^2}\\
	&\le \frac{\mulc^2}{(\mulc-1)}\frac{1}{t+1},
	\end{align*}
	where the last inequality follows since $(t-1)(t+1)=t^2-1 < t^2$.
\end{proof}

The algorithm implies that
$\lb_t  = \lb_{t-1} + \frac{\mulc}{t}(1 - \scc_t)$, for all $1 \le t \le \nit$.
Therefore,
\[
\verr_{t-1} - \verr_t
= (\amt - \lb_{t-1}) - (\amt - \lb_t)
= \lb_t - \lb_{t-1}
= \frac{\mulc}{t}(1 - \scc_t).
\]
This implies
\[
\E[1 - \scc_t]
= \E\left[\frac{t(\verr_{t-1} - \verr_t)}{\mulc} \right]
= \frac{t(\E \verr_{t-1} - \E \verr_t)}{\mulc}.
\]
Summing over $1 \le t \le \nit$,
\begin{align} \label{eq:one-bandit-proof-3}
\sum_{t=1}^\nit \E[1 - \scc_t] \nonumber
&= \sum_{t=1}^\nit \frac{t(\E \verr_{t-1} - \E \verr_t)}{\mulc} \nonumber \\
&= \frac{\E \verr_0}{\mulc} + \sum_{t=1}^{\nit-1} \frac{\E \verr_{t}}{\mulc} - \frac{\nit \E \verr_\nit}{\mulc} \nonumber\\
&\le \sum_{t=0}^{\nit-1} \frac{\E \verr_{t}}{\mulc} \nonumber \\
&\le \sum_{t=0}^{\nit-1} \frac{1}{\mulc} \frac{\mulc^2}{(\mulc-1)(t+1)} \nonumber \\
&= \frac{\mulc}{\mulc-1} \sum_{t=1}^{\nit} \frac{1}{t} \nonumber \\
&\le \frac{\mulc}{\mulc-1} (\log \nit + 1). 
\end{align}

Equations \eqref{eq:one-bandit-proof-1}, \eqref{eq:one-bandit-proof-2} and \eqref{eq:one-bandit-proof-3} conclude that:
\[
	\E \regret^{(n)} 
	= \E \left[\sum_{t=1}^\nit (\alloc_t - \amt) \right] +\E \left[\sum_{t=1}^\nit (1 - \scc_t) \right]
	\le \mulc (\log n + 1) + \frac{\mulc}{\mulc - 1} (\log n + 1)
	= \frac{\mulc^2}{\mulc-1} (\log n + 1).
\]
This is minimized at $\mulc = 2$, with a value of $4 (\log n + 1)$.

\section{Proof of Theorem~\ref{thm:multi-arm}}
\label{sec:proof:thm:multi-arm}

This is section contains a proof for the lemmas appearing in the proof
outline in Section~\ref{sec:pr-outline}.
\secref{sec:pr-def} contains a list of all definitions,
\secref{sec:p-general} presents the proof of Lemma~\ref{lem:general},
\secref{sec:p-At} presents the proof of Lemma~\ref{lem:At},
\secref{sec:p-add} presents the proof of Lemma~\ref{lem:add-bound},
\secref{sec:p-notin-T} presents the proof of Lemma~\ref{lem:notin-T},
and \secref{sec:p-slow} presents the proof of Lemma~\ref{lem:slow}.

\subsection{Table of definitions} \label{sec:pr-def}
Below is the table of all definitions.

\begin{itemize}
	\item 
	$Z'_t$: contains everything the algorithm has seen up to and including iteration $t$.
	It includes the values of $\scc_{k,t}$ for all $1 \le k \le \nban$ and $1 \le t \le n$.
	The only difference between $Z'_t$ and $Z_t$ is that $Z_t$ contains the result of the random
	coin tossed by the algorithm on iteration $t+1$, while $Z'_t$ does not.
	\item
	$\rvt_0$: $\max_{1 \le k \le \nban} \rvt_{k,0}$.
	The first iteration where all arms have positive lower bound.
	\item
	$\mo(\cdot)$: equals $\min(1, \cdot)$.
	\item 
	$n$: number of iterations the arms are invoked.
	\item
	$\nban$: number of arms.
	\item
	$\amt_1, \dots, \amt_\nban$: these parameters determine the success probability of the arms.
	Given a resource of $\alloc$, arm $k$ succeeds with probability $\min\left(1, \frac{\alloc}{\amt_k} \right)$.
	\item
	$\nal$: the number of arms that are fully allocated under the optimal allocation.
	The highest number of $i \ge 0$ such that $\sum_{j=1}^i \amt_i \le 1$.
	\item
	$\alloc_{k,t}$: the amount of resource allocated to arm $k$ on iteration $t$.
	\item
	$\scc_{k,t}$: the indicator of the success of arm $k$ on iteration $t$.
	\item
	$\lbf_{k,t}$: the deterministic lower bound of $\amt_k$,
	calculated by the algorithm at the end of iteration $t$.
	\item
	$\lbs_{k,t}$: the probabilistic lower bound of $\amt_k$,
	calculated by the algorithm at the end of iteration $t$.
	\item
	$\lb_{k,t}$: $\max(\lbf_{k,t}, \lbs_{k,t})$.
	\item
	$\mulc$: a parameter given to the algorithm, that has to get a positive value
	greater than $2$.
	It takes part in the calculation of $\add_{k,t}$.
	\item
	$(\cdot)_+$: equals $\max(0, \cdot)$.
	\item
	$\saf_{k,t}$: equals $\sum_{i \le t \colon \lbf_{k,i} > 0} (\alloc_{k,i} - \lbf_{k,i-1})_+$.
	\item
	$\add_{k,t}$: equals $\mulc \lbf_{k,t-1} \exp \left( - \frac{\saf_{k,t-1}}{\mulc \lbf_{k,t-1}} \right)$.
	Equals $\alloc_{k,t} - \lbf_{k,t-1}$ if there are sufficient resources for arm $k$ on iteration $t$.
	\item
	$\sas_{k,t}$: equals $\sum_{1 \le i \le t \colon \alloc_{k,i} \le \lbf_{k,i-1}} \alloc_{k,i}$.
	\item
	$\xs_{k,t}$: equals $\sum_{1 \le i \le t \colon \alloc_{k,i} \le \lbf_{k,i-1}} \scc_{k,i}$.
	\item
	$T$: the set of ``good'' iterations. Equals
	$\{ 1 \le t \le n \colon 0 < \lb_{k,t-1} \le \amt_k \forall 1 \le k \le \nban \}$.
	\item
	$p$: equals $\frac{c-2}{2c}$.
	\item
	$k_{1,t}, \dots, k_{\nban, t}$: the arms $1, \dots, \nban$,
	by the order which they were iterated on the loop over the arms in line~\ref{algl:k-loop} of the algorithm,
	on iteration $t$.
	\item
	$\nal_t$: the highest value of $i$ such that for all 
	$1 \le j \le i$, $\alloc_{k_{j,t}, t} = \lbf_{k_{j,t},t-1} + \add_{k_{j,t},t}$.
	\item
	$A_t$: $\{ k_{1,t}, k_{2,t}, \dots, k_{\min(\nal_t, \nal+1),t}$.
	\item
	\[ 
	B_t = \begin{cases}
	\{k'_t \} & |A_t| < \min(\nal + 1,\nban) \ \BR{AND}\  1 - \sum_{k \in A_t} \alloc_{k,t} > \lbf_{k'_t,t-1} \\
	\emptyset & \BR{otherwise}
	\end{cases}, \\
	\]
	\item
	\[ 
	C_t = \begin{cases}
	\{k'_t \} & |A_t| < \min(\nal + 1,\nban) \ \BR{AND}\  1 - \sum_{k \in A_t} \alloc_{k,t} \le \lbf_{k'_t,t-1} \\
	\emptyset & \BR{otherwise}
	\end{cases}.
	\]
	\item
	$k'_t$:$k_{|A_t|+1,t}$.
	\item
	\[
	\add'_t = \begin{cases}
	1 - \sum_{k \in A_t} \alloc_k - \lbf_{k'_t,t-1} & B_t \ne \emptyset \\
	0 & B_t = \emptyset
	\end{cases}.
	\]
	\item
	$Z_t$:the random variable that contains everything the algorithm has seen up to and just before the point it gets to see the success statuses of the arms on iteration $t+1$. It contains all the success statuses $\scc_{k,i}$ for any arm $1 \le k \le \nban$ and any iteration $1 \le i \le t$, in addition to all the randomness of the algorithm up to and including iteration $t + 1$.
	\item
	$R_t$: the regret on iteration $t$.
	\item
	$R^{(n)}$:$\sum_{t=1}^n R_t$.
	\item
	$\rvt_{k,0}$:the lowest value of $t$ for which $\lbf_{k,t} > 0$.
	\item
	$\err_{k,t}$: $\frac{\amt_k - \lbf_{k,t}}{\amt_k}$.
	\item
	$\rvt_{k,1}$: the lowest value of $t$ such that $\err_{k,t} \le p$.
	\item
	$h_{k,t}$: $\frac{(\alloc_{k,t} - \lbf_{k,t-1})_+}{\add_{k,t}}$.
	\item 
	$Z'_t$: contains everything the algorithm has seen up to and including iteration $t$.
	It includes the values of $\scc_{k,t}$ for all $1 \le k \le \nban$ and $1 \le t \le n$.
	The only difference between $Z'_t$ and $Z_t$ is that $Z_t$ contains the result of the random
	coin tossed by the algorithm on iteration $t+1$, while $Z'_t$ does not.
	\item
	$\rvt_0$: $\max_{1 \le k \le \nban} \rvt_{k,0}$.
	The first iteration where all arms have positive lower bound.
	\item
	$\mo(\cdot)$: equals $\min(1, \cdot)$.
\end{itemize}

\subsection{Proof of Lemma~\ref{lem:general}} \label{sec:p-general}
Here is a result which appears in the original work of \cite{LCS}.

\begin{lemma} \label{lem:below}
	Fix $t \in T$, and $1 \le j \le \nban$.
	Then, $\lb_{k_{j,t},t-1} \le \amt_j$,
	namely, the arm with priority $j$ on iteration $t$
	has a lower bound of at most $\amt_j$.
\end{lemma}
\begin{proof}
For any arm $k \le j$, it holds that $\lb_{k,t-1} \le \amt_k \le \amt_j$,
where the first inequality is due to the fact that $t \in T$,
and the second inequality follows from our assumption that $\amt_1 < \cdots < \amt_\nban$.
This implies that the list $\lb_{1,t-1}, \dots, \lb_{\nban,t-1}$ has at least $j$ values
lower or equal to $\amt_j$.
Therefore, if we sort the list $\lb_{1,t-1}, \dots, \lb_{\nban,t-1}$ in an 
increasing order, the value on place $j$ (counting from the start) is at most $\amt_j$.
This value is exactly $\lb_{k_{j,t},t-1}$, by definition of $k_{j,t}$.
\end{proof}

It holds that $\E \left[ \scc_{k,t} \mid Z_{t-1} \right] = \mo\para{M_{k,t}/\nu_k}$, for all $1\le k \le \nban$.
If $|A_t| = \min(\nal + 1, \nban)$, then 
\begin{align*}
\E[\regret_t \mid Z_{t-1}]
&\le \nal + \indi_{\nban > \nal} - \sum_{k=1}^\nban \mo\para{M_{k,t}/\nu_k} \\
&\le |A_t| - \sum_{k=1}^\nban \mo\para{M_{k,t}/\nu_k} \\
&\le |A_t| - \sum_{k \in A_t} \mo\para{M_{k,t}/\nu_k} \\
&= \sum_{k \in A_t} \para{1- \mo\para{M_{k,t}/\nu_k}}.
\end{align*}
Therefore, the proof follows for this case.

Assume next that $|A_t| < \min(\nal + 1, \nban)$.
Let $h \colon [0,\infty) \to \mathbb{R}$ be a function such that for all $1 \le k \le \nban$, $h(x) = 1/\nu_k$ in the range $x \in [\sum_{i=1}^{k-1} \nu_i, \sum_{i=1}^k \nu_i)$, and $h(x) = 0$ for all $x \ge \sum_{i=1}^\nban \amt_i$.
It holds that $h(x)$ is monotonic non-increasing, and its integral function $H(x) = \int_{y=0}^x h(y) dy$ satisfies that $H(1)$ is the award achieved by the optimal policy in round $t$.
Therefore, 
\begin{equation} \label{eq:lem-general-1}
\E[\regret_t \mid Z_{t-1}] = H(1) - \sum_{k=1}^\nban \mo\para{M_{k,t}/\nu_k}.
\end{equation}
Let $a = \sum_{k=1}^{|A_t|} \amt_k$ and
$b = \sum_{k=1}^{|A_t|} \amt_k + \sum_{k \in A_t} \add_{k,t} + \add'_t$.
Using equality \eqref{eq:lem-general-1},
\begin{align} \label{eq:lem-general-2}
\E[\regret_t \mid Z_{t-1}]
&\le \left(H(a) - \sum_{k \in A_t} \mo\para{M_{k,t}/\nu_k} \right) \\
&+\left( H(b) - H(a) \right) \label{eq:genm2} \\
&+\left( H(1) - H(b) - \mo\left( \frac{\alloc_{k'_t,t}}{\amt_{k'_t}} \right) \right) \label{eq:genm3}.
\end{align}
We will bound each of these three terms separately.

The right hand side in \eqref{eq:lem-general-2} is bounded by
\begin{equation} \label{eq:lem-general-3}
H\para{\sum_{k=1}^{|A_t|} \nu_k} - \sum_{k\in A_t} \mo\para{M_{k,t}/\nu_k}
= \sum_{k\in A_t} (1 - \mo\para{M_{k,t}/\nu_k}).
\end{equation}

We proceed to bounding the quantity in \eqref{eq:genm2}.
Lemma~\ref{lem:below} implies that any $k \in A_t \cup \{ k'_t \} = \{ k_{1,t}, \dots, k_{|A_t|+1,t} \}$ 
satisfies $\lbf_{k,t-1} \le \amt_{|A_t|+1}$.
Therefore,
\begin{align}
H(b) - H(a)
&\le \int_a^b \frac{1}{\amt_{|A_t|+1}} \nonumber \\
&= \frac{\sum_{k \in A_t} \add_{k,t} + \add'_t}{\amt_{|A_t|+1}} \nonumber \\
&\le \sum_{k \in A_t} \frac{\add_{k,t}}{\lbf_{k,t-1}} + \frac{\add'_t}{\lbf_{k'_t,t-1}}. \label{eq:lem-general-4}
\end{align}

Lastly, bound the quantity in \eqref{eq:genm3}.
Lemma~\ref{lem:below} implies that
\[
	\sum_{k \in A_t} \lbf_{k,t-1} 
	= \sum_{j=1}^{|A_t|} \lbf_{k_{j,t},t-1}
	\le \sum_{k=1}^{|A_t|} \amt_j.
\]
This implies that 
\begin{equation} \label{eq:lem-general-6}
\sum_{k \in A_t} \alloc_{k,t} 
= \sum_{k \in A_t} (\lbf_{k,t-1} + \add_{k,t})
\le \sum_{k=1}^{|A_t|} \amt_k + \sum_{k \in A_t} \add_{k,t}.
\end{equation}
We will show that 
\begin{equation} \label{eq:lem-general-5}
b \ge 1 - \min(\alloc_{k'_t,t}, \lbf_{k'_t,t-1}).
\end{equation}
First, assume that $B_t = \emptyset$. Inequality~\eqref{eq:lem-general-6} implies that
\[
b
\ge \sum_{k \in A_t} \alloc_{k,t}
= 1 - \alloc_{k'_t, t}
= 1 - \min(\alloc_{k'_t,t}, \lbf_{k'_t,t-1}).
\]
If $B_t \ne \emptyset$, then 
$\add_t' = 1 - \sum_{k \in A_t} \alloc_k - \lbf_{k'_t,t-1}$, and
\[
b
\ge \sum_{k \in A_t} \alloc_{k,t} + \add'_t
= 1 - \lbf_{k'_t,t-1}
= 1 - \min(\lbf_{k'_t,t-1}, \alloc_{k'_t,t}),
\]
which concludes the proof of Equation~\eqref{eq:lem-general-5}.
This implies that 
\begin{align} 
H(1) - H(b)
&= \int_{x=b}^1 h(x) dx \nonumber \\
&\le (1 - b) h(b) \nonumber \\
&\le \min(\lbf_{k'_t,t-1}, \alloc_{k'_t,t}) h(b). \label{eq:lem-general-7}
\end{align}

If $|A_t| = \nal$, then 
$b \ge \sum_{k=1}^{\nal} \nu_k$.
Therefore, 
$h(b) \le \frac{1}{\amt_{\nal+1}}$,
which implies, together with Equation~\eqref{eq:lem-general-7}, that 
\begin{equation} \label{eq:lem-general-8}
H(1) - H(b)
\le \min(\lbf_{k'_t,t-1}, \alloc_{k'_t,t}) \frac{1}{\amt_{\nal+1}}.
\end{equation}

If $|A_t| < \nal$, then, we know that 
$ \lb_{k'_t,t-1} \le \amt_{|A_t|+1} \le \amt_{\nal}$,
which implies, together with equation~\eqref{eq:lem-general-5} that
\[
b
\ge 1 - \min(\lbf_{k'_t,t-1}, \alloc_{k'_t,t})
\ge 1 - \lbf_{k'_t,t-1}
\ge 1 - \amt_\nal.
\]
Therefore,
\[
h(b)
\le h(1 - \amt_\nal)
\le \frac{1}{\amt_\nal}.
\]
This implies, together with Equation~\eqref{eq:lem-general-7}, that
\begin{equation} \label{eq:lem-general-9}
H(1) - H\left( b \right)
\le \min(\lbf_{k'_t,t-1}, \alloc_{k'_t,t}) \frac{1}{\amt_{\nal}}.
\end{equation}
Additionally,
\begin{equation} \label{eq:lem-general-10}
\mo\left( \alloc_{k'_t,t}/\amt_{k'_t} \right)
\ge \mo\left( \min(\lbf_{k'_t,t-1}, \alloc_{k'_t,t})/\amt_{k'_t} \right)
= \min(\lbf_{k'_t,t-1}, \alloc_{k'_t,t})/\amt_{k'_t}.
\end{equation}
Equations \eqref{eq:lem-general-8}, \eqref{eq:lem-general-9} and \eqref{eq:lem-general-10} imply that
\begin{equation} \label{eq:lem-general-11}
H(1) - H(b) - \mo\left( \alloc_{k'_t,t}/\amt_{k'_t} \right)
\le \begin{cases}
\min(\lbf_{k'_t,t-1}, M_{k'_t,t}) (1/\nu_{\nal+1}-1/\nu_{k'_t}) & |A_t| = \nal \\
\min(\lbf_{k'_t,t-1}, M_{k'_t,t}) (1/\nu_{\nal}-1/\nu_{k'_t}) & |A_t| < \nal
\end{cases}.
\end{equation}
Equations \eqref{eq:lem-general-2}, \eqref{eq:lem-general-3}, \eqref{eq:lem-general-4} and \eqref{eq:lem-general-11} conclude the proof.

\subsection{Proof of Lemma~\ref{lem:At}} \label{sec:p-At}

We present the lemmas required for the proof, together with an intuition for the proof.
Define the \emph{error} of arm $k$ on iteration
$t$ by $\err_{k,t} = \frac{\amt_k - \lbf_{k,t}}{\amt_k}$. We would
like to bound the convergence rate of $\err_{k,t}$ to $0$.
The rate is in terms of the number of iterations: how many iterations it takes for $\err_{k,t}$ to get below some threshold? Optimally, when there are sufficient resources, arm $k$ is allocated with $\lbf_{k,t-1} + \add_{k,t}$ resources. However, if there are insufficient resources and one allocates $\alloc_{k,t}\le \lbf_{k,t-1}$, then one knows that $\lbf_{k,t}$ will not improve, namely, $\lbf_{k,t} = \lbf_{k,t-1}$. Hence, one should not count iterations when $\alloc_{k,t} \le \lbf_{k,t-1}$ while estimating the number of iterations it takes for $\lbf_{k,t}$ to get below some threshold. One might ask: if iterations where $\alloc_{k,t} = \lbf_{k,t-1}+\add_{k,t}$ are counted as $1$ and iterations where $\alloc_{k,t} \le \lbf_{k,t-1}$ are counted as $0$, how should iterations where $\lbf_{k,t-1} < \alloc_{k,t} < \lbf_{k,t-1}+\add_{k,t}$ be counted? The answer is that these iterations are counted as $\left(\alloc_{k,t} - \lbf_{k,t-1}\right)/\add_{k,t}$. Combining everything together, every iteration $t$ is counted as $h_{k,t} := \frac{(\alloc_{k,t} - \lbf_{k,t-1})_+}{\add_{k,t}}$, where $(\cdot)_+ = \max(0,\cdot)$. In particular, every iteration that $k$ is case A is counted as $1$, every iteration that $k$ is case B is counted as some positive number less than $1$, and iterations that $k$ is case $C$ are counted as $0$.

Define $\rvt_{k,0}$ as the lowest value of $t$ for which $\lbf_{k,t} > 0$ (equivalently, the last iteration that $k$ is allocated according to case I),
and define $p = \frac{\mulc - 2}{2 \mulc}$.
We start by bounding the number of iterations (weighted by $h_{k,t}$)
that pass from $\rvt_{k,0}$ up to the point that the error $\err_{k,t}$ is at most $p$
(equivalently, from the first iteration that $\lbf_{k,t}>0$ to the first iteration that $\lbf_{k,t} \ge (1 - p) \amt_k$).
This number is bounded by $O\left( \log \frac{(1-p)\amt_k}{\lbf_{k,\rvt_{k,0}}}\right)$,
which implies that the estimate $\lbf_{k,t}$ grows exponentially fast in the beginning.
\begin{lemma} \label{lem:start}
	Fix $1 \le k \le \nban$.
	Fix $\gamma \le (1 - p) \amt_k$.
	Let $\rvt$ be the first iteration $t$ that $\lbf_{k,t} \ge \gamma$. Then
	\[
	\E \left[ \sum_{\rvt_{k,0} < t \le \rvt} h_{k,t} \middle\vert \rvt_{k,0},\ \lbf_{k,\rvt_{k,0}} \right]
	\le C \left(\log \frac{\gamma}{\lbf_{k,\rvt_{k,0}}} \right)_+,
	\]
	where $C > 0$ is some constant, depending only on $\mulc$.
\end{lemma}
In order to give an intuitive reason to this exponential growth, recall the definition of $\add_{k,t}$ in \figref{alg:multi-arm}. Fix some $t$ and assume that $\lbf_{k,t-1} \le (1-p)\amt_k$ and $\add_{k,i}/\amt_k \le p/2$ for all $i \le t$. Then, for all $i \le t$,
\[
\mathbb{E}\left[1-\scc_{k,i}\middle| \alloc_{k,i}\right]
\ge 1 - \frac{\alloc_{k,i}}{\amt_k}
\ge 1 - \frac{\lbf_{k,i-1}+\add_{k,i}}{\amt_k}
\ge p/2.
\]
This implies that
\[
\mathbb{E}\left[\frac{\lbf_{k,t-1}}{\saf_{k,t-1}}\right]
= \mathbb{E}\left[\frac{\sum_{i=\rvt_{k,0}}^{t-1} \lbf_{k,i} - \lbf_{k,i-1}}{\sum_{i=\rvt_{k,0}}^{t-1} \left( \alloc_{k,i} - \lbf_{k,i-1} \right)_+}\right]
= \mathbb{E}\left[\frac{\sum_{i=\rvt_{k,0}}^{t-1} \left( \alloc_{k,i} - \lbf_{k,i-1} \right)_+(1-\scc_{k,i})}{\sum_{i=\rvt_{k,0}}^{t-1} \left( \alloc_{k,i} - \lbf_{k,i-1} \right)_+}\right]
\ge p/2.
\]
Hence,
\[
\frac{\add_{k,t}}{\lbf_{k,t-1}} = c \exp\left(-\frac{\saf_{k,t-1}}{c\lbf_{k,t-1}} \right) = \Omega(1)
\]
with high probability, which implies that
\[
\mathbb{E} \left[\frac{\lbf_{k,t}}{\lbf_{k,t-1}} \right]
= \mathbb{E} \left[\frac{\lbf_{k,t-1} + (1 - \scc_{k,t})\left(\alloc_{k,t} - \lbf_{k,t-1}\right)_+}{\lbf_{k,t-1}} \right]
= 1 + \mathbb{E}\left[1-\scc_{k,t} \right] \frac{h_{k,t} \add_{k,t}}{\lbf_{k,t-1}}
= 1 + \Omega(1) h_{k,t}.
\]
This implies that $\lbf_{k,t}$ is indeed growing exponentially fast (with respect to $h_{k,t}$), however, recall we assumed that $\add_{k,i} /\amt_k \le p/2$ for all $i\le t$. This assumption was made in order to ensure that $\mathbb{E}\left[1 - \scc_{k,i} \right]$ is sufficiently large, so that $\scc_{k,i}=0$ sufficiently often. However, one does not need this assumption: if $\scc_{k,i} = 1$ for a sufficiently large constant number of times, $\add_{k,t}$ shrinks and gets below $p/2$. The formal claim is proved inductively using a potential function.

Define by $\rvt_{k,1}$ the first iteration that $\lbf_{k,t} \ge (1-p) \amt_k$, or, equivalently, the first iteration that $\epsilon_{k,t} \le p$. The next lemma bound the number of iterations that pass from $\rvt_{k,1}$ until $\epsilon_{k,t}\le \eta$ by $O(1/\eta)$ plus another term which depends on $\saf_{k,\rvt_{k,1}}$, for any $\eta>0$. 
\begin{lemma} \label{lem:err-small}
	Fix an integer $k$, $1 \le k \le \nban$.
	Fix some number $0 < \eta < 1$.
	Let $\rvt$ be the first iteration $t$ such that $\err_t \le \eta$.
	Then, there exists a numerical constant $C > 0$ depending only on $\mulc$, such that
	\[
	\E \left[ \sum_{t=\rvt_{k,1}+1}^{\rvt} h_{k,t}
	\middle\vert \rvt_{k,1},\, \saf_{k,\rvt_{k,1}} \right]
	\le C \left( \exp\left( \frac{2 \saf_{k,\rvt_{k,1}}}{\mulc \amt_k (1-p)} \right) + \frac{1}{\eta} \right).
	\]
\end{lemma}
One would expect the term $O(1/\eta)$, since the estimate $\lbf_{k,t}$ behaves as the estimante $\lb_k$ in the single armed problem, which requires roughly $O(1/\eta)$ iterations to bet below $\eta$. However, since the algorithm for the multi armed setting involves some complications not existant in the single armed algorithm, the proof is obtained by induction using a potential function.

We add two comments. Firstly, one may ask why the sum in \lemref{lem:err-small} begins with $\rvt_{k,1}+1$ instead of $\rvt_{k,0}$ or $1$. Since the construction of $\add_{k,t}$ uses $\lbf_{k,t-1}$ to approximate $\amt_k$, one requires this approximation to be accurate in order for the lemma to hold. Secondly, note the term $\saf_{k,\rvt_{k,1}}$ in the bound in \lemref{lem:err-small}. If this term is very large, $\add_{k,t}$ would be small, and the estimate $\lbf_{k,t}$ would not be able to improve fast. However, one can bound this term. As explained in the intuition for \lemref{lem:start}, $\add_{k,t}/\lbf_{k,t}$ is expected not to be low in the beginning, which implies that $\saf_{k,t}$ is not high. We present the lemma which bounds this term. The formal proof is by induction using a potential function, and requires some case analysis.

\begin{lemma} \label{lem:k-sqr}
	Fix some arm $0 \le k \le \nban$.
	Then, for some constant $C > 0$ depending only on $c$, 
	\[
	\E\left[ \exp\left( \frac{2 \saf_{k,\rvt_{k,1}}}{\mulc \amt_k (1-p)} \right) \right] \le C.
	\]

\end{lemma}

\secref{sec:At-first} and \secref{sec:At-aux} present auxiliary lemmas,
\secref{sec:At-start} presents the proof of \lemref{lem:start},
\secref{sec:At-err} presents the proof of \lemref{lem:err-small},
\secref{sec:At-sqr} presents the proof of \lemref{lem:k-sqr}
and \secref{sec:At-conc} concludes the proof.

\subsubsection{Lemma~\ref{lem:first-bd}} \label{sec:At-first}
This lemma bounds the number of iterations before $\lbf_{k,t} > 0$, for any arm $k$.

\begin{lemma} \label{lem:first-bd}
	For any $1 \le k \le \nban$,
	\[
	\E\left[ \log_2 \frac{\amt_k}{\lbf_{k,\rvt_{k,0}}} \right] \le \max(2, \log_2 (\amt_k \nban)) + 1.
	\]
	Additionally
	\[
	\E[\rvt_0] \le \max\left( 1, \left\lceil \log_2 \frac{1}{\amt_1} + 3 \right\rceil \right).
	\]
\end{lemma}

Fix $k$, $1 \le k \le \nban$.
Let $t' = \max \left( \left\lceil \log_2\left( \frac{1}{\nban\amt_k} \right) + 1 \right\rceil, 0 \right)$.
At iteration $t'$ it holds that
\[
\frac{1}{\nban 2^{t'-1}} 
\le \frac{1}{\nban 2^{\log_2(1/(\nban \amt_k))}}
= \amt_k. 
\]
Therefore, for any $t > t'$, assuming that $\lbf_{k,t-1} = 0$ it holds that
\begin{align*}
\Pr[\lbf_{k,t} > 0 \mid \lbf_{k,t-1} = 0]
&= \Pr[\scc_{k,t} = 0 \mid \lbf_{k,t-1} = 0]
= 1 - \mo\left(\alloc_{k,t}/\amt_k\right)\\
&= 1- \frac{1}{\nban 2^{t-1} \amt_k}
\ge 1- \frac{1}{2}\frac{1}{\nban 2^{t'-1} \amt_k}
\ge 1/2.
\end{align*}
Therefore, for any iteration $t > t'$, $\Pr[\rvt_{k,0} = t \mid \rvt_{k,0} > t-1] \ge 1/2$.
Therefore, given that $\rvt_{k,0} > t'$, $\E[\rvt_{k,0}] - t'$ equals at most the expectancy of a geometric random variable with parameter $1/2$, which implies that 
\[
\E[\rvt_{k,0}] \le t' + 2 
\le \max\left(\log_2\left( \frac{1}{\nban\amt_k} \right) + 2,0 \right) + 2.
\]

We calculate the expected value of $\log_2 \frac{\amt_k}{\lbf_{k,\rvt_{k,0}}}$. It holds that
\begin{align*}
\E\left[\log_2 \frac{\amt_k}{\lbf_{k,\rvt_{k,0}}} \right] 
&= \E \left[\log_2 \left(\amt_k \nban 2^{\rvt_{k,0}-1}\right) \right] 
= \E \left[\log_2 (\amt_k \nban) + \rvt_{k,0}-1 \right] \\
&\le \log_2 (\amt_k \nban)  + \max(\log_2 \frac{1}{\nban \amt_k} + 2, 0) + 1  
=  \max(2, \log_2 (\amt_k \nban)) + 1.
\end{align*}

Lastly, let $t' = \max\left( 0, \left\lceil \log_2 \frac{1}{\amt_1} + 2 \right\rceil \right)$.
For any $t \ge t'$ it holds that
\[
\frac{1}{\nban 2^{t-1}} 
\le \frac{1}{\nban 2^{\log_2 \frac{1}{\amt_1} + 1}}
= \frac{\amt_1}{2\nban}.
\]
This implies that for any $1 \le k \le \nban$, for any $t \ge t'$, it holds that whenever $\lbf_{k,t-1} = 0$,
$\alloc_{k,t} = \frac{1}{\nban 2^{t-1}}  \le \frac{\amt_1}{2\nban}$.
Therefore, it holds that $\scc_{k,t} > 0$ with probability at most $\frac{\alloc_{k,t}}{\amt_k} \le \frac{1}{2\nban}$.
Therefore, given that $t \ge t'$ and that $\rvt_0 > t-1$, the probability that there exists $1 \le k \le \nban$ such that $\lbf_{k,t-1} = 0$ and $\scc_{k,t} > 0$, is at most $\sum_{k \colon \lbf_{k,t-1} = 0} \frac{1}{2\nban} \le 1/2$.
This implies that for any $t \ge t'$, given that $\rvt_0 > t - 1$, it holds that with probability at least $1/2$, $\rvt_0 = t$.
This implies that conditioned on $\rvt_0 > t'-1$, it holds that $\rvt_0 - (t'-1)$ is bounded by a geometric random variable with parameter 2. Therefore,
\[
\E \left[ \rvt_0 - (t'-1) \middle\vert \rvt_0 > t'-1 \right]
\le 2.
\]
Thus,
\[
\E\left[ \rvt_0 \right] \le t' - 1 + 2 = t'+1.
\]

\subsubsection{Lemma~\ref{lem:aux-At-add}} \label{sec:At-aux}
\begin{lemma} \label{lem:aux-At-add}
	There exist constants $C, C' > 0$, depending only on $\mulc$, 
	such that for any $k$, $1 \le k \le \nban$,
	\begin{equation} \label{eq:cor-start-bd-0}
	\E \left[ \sum_{t = \rvt_{k,0}+1}^{\rvt_{k,1}} \frac{(\alloc_{k,t} - \lbf_{k,t-1})_+}{\add_{k,t}} \right]
	\le C \log \nban + C'.
	\end{equation}
\end{lemma}

	Start by assuming that $\amt_k \le \frac{1}{1-p}$.
	From Lemma~\ref{lem:start}, it holds that there exist a constants $c_1, c_1' > 0$, such that for any values of $\rvt_{k,0}$ and $\lbf_{k,\rvt_{k,0}}$,
	\begin{equation} \label{eq:cor-start-bd-1}
	\E \left[ \sum_{t = \rvt_{k,0}+1}^{\rvt_{k,1}} \frac{(\alloc_{k,t} - \lbf_{k,t-1})_+}{\add_{k,t}} \middle\vert \rvt_{k,0},\ \lbf_{k,\rvt_{k,0}} \right]
	\le c_1 \log \frac{(1-p)\amt_k}{\lbf_{k,\rvt_{k,0}}} + c_1'.
	\end{equation}
	From Lemma~\ref{lem:first-bd}, 
	there exists a constant $c_2$ such that 
	\begin{equation} \label{eq:cor-start-bd-4}
	\E \log_2 \frac{\amt_k}{\lbf_{k,\rvt_{k,0}}} 
	\le (\log_2 \nban + \log_2 \amt_k)_+ + c_2
	\le \log_2 \nban + \log_2 \frac{1}{1-p} + c_2.
	\end{equation}
	Together, Equalities~\eqref{eq:cor-start-bd-1} and \eqref{eq:cor-start-bd-4} conclude the proof for the case $\amt_k \le \frac{1}{1-p}$.
	
	Next, assume that $\amt_k > \frac{1}{1-p}$.
	Let $\rvt$ be the first iteration $t$ such that $\lbf_{k,t} = 1$.
	For any $t$, $1 \le t \le n$, if $\lbf_{k,t-1} = 1$, then $\frac{(\alloc_{k,t} - \lbf_{k,t})_+}{\add_{k,t}} = 0$.
	This, together with Lemma~\ref{lem:first-bd}, imply that
	\begin{align} 
	&\E \left[ \sum_{t = \rvt_{k,0}+1}^{\rvt_{k,1}} \frac{(\alloc_{k,t} - \lbf_{k,t-1})_+}{\add_{k,t}} \middle\vert \rvt_{k,0},\ \lbf_{k,\rvt_{k,0}} \right] \nonumber \\
	&= \E \left[ \sum_{t = \rvt_{k,0}+1}^{\rvt} \frac{(\alloc_{k,t} - \lbf_{k,t-1})_+}{\add_{k,t}} \middle\vert \rvt_{k,0},\ \lbf_{k,\rvt_{k,0}} \right] \nonumber \\
	&\le c_1 \log \frac{1}{\lbf_{k,\rvt_{k,0}}} + c_1'. \label{eq:cor-start-bd-2}
	\end{align}
	Lemma~\ref{lem:first-bd} implies that
	\begin{equation} \label{eq:cor-start-bd-3}
	\E\left[\log \frac{1}{\lbf_{k,\rvt_{k,0}}}\right] 
	= \E\left[\log \frac{\amt_k}{\lbf_{k,\rvt_{k,0}}}\right]  - \log \amt_k
	\le \log \nban + \log \amt_k  + c_2 - \log \amt_k
	= \log \nban + c_2.
	\end{equation}
	Equations~\eqref{eq:cor-start-bd-2} and \eqref{eq:cor-start-bd-3} suffice to complete the proof.

\subsubsection{Proof of Lemma~\ref{lem:start}} \label{sec:At-start}

Fix some integer $k$, $1 \le k \le \nban$.
Given any $t \ge \rvt_{k,0}$, define
\[ w_t = \exp \left( \frac{\saf_{k,t}}{\lbf_{k,t} \mulc} \right). \]

We will prove by induction on $m \ge 0$ that for all $t > \rvt_{k,0}$, whenever $\lbf_{k,t-1} \le 2 \gamma$, it holds that
\[
\E \left[ \sum_{t \le i < t+m \colon x_{t-1} < \gamma} 
\frac{(\alloc_{k,i} - \lbf_{k,i-1})_+}{\add_{k,i}} \middle\vert Z_{t-1} \right]
\le \phi(\lbf_{k,t-1}, w_{t-1}),
\]
where 
\[
\phi(u, w)= w + \alpha_1 \ln \frac{2\gamma}{u} + \alpha_2 (c_2 - w)_+,
\]
and
\begin{align*}
c_1 &= \frac{2 \mulc}{p},\\
c_2 &= c_1 + 2,\\
c_3 &= \max(4c \log (2c) ,2 c_2, 12/p, \exp(\frac{24/p+1}{\mulc})),\\
\alpha_1 &= \frac{6/p + \mulc \log c_3 + 1}{\log (1 + \mulc/c_3)},\\
\alpha_2 &= 2.
\end{align*}

For the base of induction, assume that $m = 0$. Since we assumed that $\lbf_{k,t-1} \le 2 \gamma$, the potential function is non-negative.

For the step of induction, assume that $m > 0$. 
Fix some $t > \rvt_{k,0}$, and fix $Z_{t-1}$.
Assume that $\lbf_{k,t-1} < \gamma$, otherwise the bound is trivially correct.
Denote shortly $u = \lbf_{k,t-1}$, $s = \saf_{k,t-1}$ and $w = w_{t-1}$.
Let $h$ be the value such that
\[ (\alloc_{k,t} - \lbf_{k,t-1})_+ = h \add_{k,t} = \frac{h \mulc u}{w}. \] 
Let $q = \Pr[\scc_{k,t} = 0 \mid Z_{t-1}]$.
It holds that
\[ \saf_{k,t} = \saf_{k,t-1} + h \add_{k,t} = s + \frac{h \mulc u}{w}, \]
and
\[
\lbf_{k,t} = \begin{cases}
u & \scc_{k,t} = 1 \\
\alloc_{k,t} = u + \frac{h \mulc u}{w} & \scc_{k,t}=0
\end{cases}.
\]

Let $u^{(0)}$ and $u^{(1)}$ be the corresponding values of $\lbf_{k,t}$ given the value of $\scc_{k,t}$, namely
\[
u^{(0)} = u + \frac{h \mulc u}{w},\quad
u^{(1)} = u.
\]
Let $w^{(0)}$ and $w^{(1)}$ be defined similarly, and denote $s^{(0)} = s^{(1)} = s$.
It remains to prove the following inequality:
\begin{equation} \label{eq:lem-start-recursive}
q \phi(u^{(0)}, w^{(0)}) + (1-q) \phi(u^{(1)}, w^{(1)}) + h \le \phi(u, w).
\end{equation}
We use the following shorthand definitions:
\[
\phi^{(0)} = \phi(u^{(0)}, w^{(0)}), 
\phi^{(1)} = \phi(u^{(1)}, w^{(1)}),
\phi(u, w) = \phi.
\]

We proceed by proving some inequalities which will be required in the proof.

\begin{proposition} \label{prop:x-log-x}
	For all $a > 1$, and all $y \ge 2 a \log a$, it holds that $y \ge a \log y$.
\end{proposition}

\begin{proof}
	Start by setting $y = 2 a \log a$, and $b + 1 = \log a$.
	It holds that 
	\[
	y
	= 2 a \log a
	= a (\log a + b + 1)
	\ge a(\log a + \log (b+1) + 1)
	= a(\log a + \log \log a + 1)
	= a \log(2a \log a)
	= a \log y,
	\]
	using the inequality $x \ge \log (x+1)$ for all $x \in \mathbb{R}$.
	Next, note that the function $y - a \log y$ monotonic increasing in $y$ for all $y \ge a$,
	therefore the inequality indeed holds for all $y \ge 2 a \log a$.
\end{proof}

\begin{lemma} \label{lem:simple-ineq}
	Let $u, s, w$ be defined as above.
	The following inequalities hold:
	\begin{enumerate}
		\item \label{itm:simple-ineq-1}
		If $w \le c_3$, then
		\[
		- (s/u + 1) (\alpha_2 - 1) + \alpha_1 \log (1 + c/c_3) \ge 6/p.
		\]
		\item \label{itm:simple-ineq-2}
		If $w \ge c_3$, then
		\[
		w  \ge c_2 + s/u.
		\]
		\item \label{itm:simple-ineq-3}
		If $w \ge c_3$, then $\frac {p w}{4} \ge 3$.
		\item \label{itm:simple-ineq-4}
		If $w \ge c_3$, then $\frac{s}{u} \ge \frac{24}{p} + 1$.
	\end{enumerate}
\end{lemma}

\begin{proof}
	Note that $c \log w = \frac{s}{u}$.
	Start with proving item~\ref{itm:simple-ineq-1}.
	Whenever $w \le c_3$, it holds that
	\begin{align*}
	- (s/u + 1) (\alpha_2 - 1) + \alpha_1 \log (1 + c/c_3)
	&= - (\mulc \log w + 1) + \alpha_1 \log (1 + c/c_3) \\
	&\ge -(\mulc \log c_3 + 1) + 6/p + 3 \mulc \log c_3 \\
	&= 6/p.
	\end{align*}
	
	Next, we prove item~\ref{itm:simple-ineq-2}.
	It is clear that $w/2 \ge c_3 / 2 \ge c_2$.
	Proposition~\ref{prop:x-log-x} implies that for all $w \ge c_3 \ge 2 (2 \mulc) \log (2 \mulc)$
	it holds that $w \ge 2\mulc \log w = 2 s/u$ by substituting $a = 2 \mulc$.
	Therefore,
	\[
	w = w/2 + w/2 \ge c_2 + s/u
	\]
	as required.
	
	Items~\ref{itm:simple-ineq-3} and \ref{itm:simple-ineq-4} trivially follow from the definition of $c_3$, and the equality
	$c \log w = \frac{s}{u}$.
\end{proof}

\begin{lemma} \label{lem:start-w}
	Let $w$, $w^{(0)}$, $w^{(1)}$, $s$, $u$ and $h$ be defined as above.
	Then 
	\begin{itemize}
		\item 
		$w^{(0)} \le w \le w^{(1)}$.
		\item
		$w + h \le w^{(1)} \le w + 2h$.	
		\item
		\[
		w - (s/u-1) h \le w^{(0)} \le \max\left(w/2, w - \frac{(s/u-1)h}{2(1 + \mulc /w)} \right).
		\]
	\end{itemize}
\end{lemma}

\begin{proof}
	The upper bound for $w^{(1)}$ is as follows:
	\begin{equation}
	w^{(1)}
	= e^{\frac{s^{(1)}}{\mulc u^{(1)}}}
	= \exp\left(\frac{s + \frac{h \mulc u}{w}}{\mulc u}\right)
	= w e^{h/w}
	\le w (1 + 2h/w)
	= w + 2h,
	\end{equation}
	using the inequality $\exp(y) \le 1 + 2y$ for all $0 \le y \le 1$.
	The lower bound is calculated similarly:
	\begin{equation}
	w^{(1)}
	= w e^{h/w}
	\ge w (1 + h/w)
	= w + h,
	\end{equation}
	using the inequality $e^y \ge 1 + y$ for all $y \in \mathbb{R}$.
	
	Next, we calculate the inequalities regarding $w^{(0)}$:
	\begin{align} 
	w^{(0)}
	&= \exp \left( \frac{s^{(0)}}{\mulc u^{(0)}} \right) \nonumber \\
	&= \exp \left( \frac{s + c h u / w}{\mulc u + \mulc^2 h u /w} \right) \nonumber \\
	&= \exp \left( \frac{s/(cu) + h / w}{1 + \mulc h /w} \right) \nonumber \\
	&= \exp \left( \frac{s/(cu) + (s/u)(h/w) - (s/u-1)(h/w)}{1 + \mulc h /w} \right) \nonumber \\
	&= \exp \left(s/(cu) - \frac{(s/u-1)(h/w)}{1 + \mulc h /w} \right) \nonumber \\
	&= w \exp \left(- \frac{(s/u-1)(h/w)}{1 + \mulc h /w} \right) \nonumber \\
	&\ge w \exp \left(- (s/u-1)(h/w) \right) \nonumber \\
	&\ge w(1 - (s/u-1)(h/w)) \label{eq:start-0} \\
	&\ge w - (s/u-1)h, \nonumber
	\end{align}
	where \eqref{eq:start-0} follows from the inequality $e^y \ge 1 + y$, for all $y \in \mathbb{R}$.
	
	Before calculating the upper bound on $w^{(0)}$, we first show that $s \ge u$, by proving that $\saf_{k,t} \ge \lbf_{k,t}$, for all $t \ge \rvt_{k,0}$.
	For $t = \rvt_{k,0}$, it holds that $\saf_{k,t} = \lbf_{k,t} = \alloc_{k, t}$.
	For $t > \rvt_{k,0}$ it holds that
	\[ 
	\lbf_{k,t} - \lbf_{k,t-1}
	= (\alloc_{k,t} - \lbf_{k,t-1})_+ (1 - \scc_{k,t})
	\le (\alloc_{k,t} - \lbf_{k,t-1})_+
	= \saf_{k,t} - \saf_{k,t-1}.
	\]
	Next, we proceed to bounding $w^{(0)}$.
	\begin{align}
	w^{(0)}
	&= w \exp \left(- \frac{(s/u-1)(h/w)}{1 + \mulc h /w} \right) \nonumber \\
	&\le w \exp \left(- \frac{(s/u-1)(h/w)}{1 + \mulc /w} \right) \nonumber \\
	&\le w \max\left(1/2, 1 - \frac{(s/u-1)(h/w)}{2(1 + \mulc /w)} \right) \label{eq:lem-start-6} \\
	&= \max\left(w/2, w - \frac{(s/u-1)h}{2(1 + \mulc /w)} \right) \nonumber,
	\end{align}
	where \eqref{eq:lem-start-6} follow from the inequality $e^{-x} \le 1-x/2$ whenever $0 \le x \le 1$ and $e^{-x} \le 1/2$ whenever $x \ge 1$. It cannot happen that $\frac{(s/u-1)(h/w)}{1 + \mulc /w} < 0$ since, as we explained $s \ge u$, and this confirms that $w^{(0)} \le w$.
\end{proof}

\begin{lemma} \label{lem:ineq-start-c3}
	If $w \le c_3$ then
	\[
	\phi^{(1)} - \phi^{(0)} \ge \frac{6h}{p}.
	\]
\end{lemma}

\begin{proof}
	We start with an inequality:
	\begin{equation} \label{eq:start-c3-1}
	\log \frac{1}{u^{(1)}} - \log \frac{1}{u^{(0)}}
	= \log \frac{u^{(0)}}{u^{(1)}}
	= \log \frac{u + h \mulc u/w}{u}
	= \log \left( 1 + h \mulc /w \right)
	\ge h \log (1 + \mulc/w)
	\ge h \log (1 + \mulc/c_3),
	\end{equation}
	using the inequality $\log (1 + \alpha x) \ge \alpha \log (1+x)$, for $x \ge 0$ and $0 \le \alpha \le 1$.
	
	Next, we prove
	\begin{equation} \label{eq:start-c3-2}
	(c_2 - w^{(0)})_+ - (c_2 - w^{(1)})_+ \le w^{(1)} - w^{(0)}.
	\end{equation}
	Lemma~\ref{lem:start-w} states that $w^{(1)} \ge w^{(0)}$.
	Whenever $w^{(1)} \le c_2$ it holds
	\[
	(c_2 - w^{(0)})_+ - (c_2 - w^{(1)})_+ = w^{(1)} - w^{(0)}.
	\]
	And whenever $w^{(1)} \ge c_2$ it holds
	\[
	(c_2 - w^{(0)})_+ - (c_2 - w^{(1)})_+ = (c_2 - w^{(0)})_+ \le w^{(1)} - w^{(0)},
	\]
	which confirms the validity of inequality~\eqref{eq:start-c3-2}.
	
	Thus,
	\begin{align}
	\phi^{(1)} - \phi^{(0)}
	&=w^{(1)} + \alpha_2 (c_2 - w^{(1)})_+ - w^{(0)} - \alpha_2 (c_2 - w^{(0)})_+ + \alpha_1\left(\log \frac{1}{u^{(1)}} - \log \frac{1}{u^{(0)}} \right) \nonumber\\
	&\ge (w^{(0)} - w^{(1)})(\alpha_2 - 1) + h \alpha_1 \log (1 + c/c_3) \label{eq:lem-start-phi0-phi1-1}\\
	&\ge - (s/u + 1) h (\alpha_2 - 1) + h \alpha_1 \log (1 + c/c_3) \label{eq:lem-start-phi0-phi1-2}\\
	&\ge \frac{6h}{p}, \label{eq:lem-start-small-k-2}
	\end{align}
	where line~\eqref{eq:lem-start-phi0-phi1-1} follows from inequalities~\eqref{eq:start-c3-1} and \eqref{eq:start-c3-2},
	line~\eqref{eq:lem-start-phi0-phi1-2} follows from Lemma~\ref{lem:start-w},
	and line~\eqref{eq:lem-start-small-k-2} follows from Lemma~\ref{lem:simple-ineq}.\ref{itm:simple-ineq-1}.
\end{proof}

We start by proving inequality~\eqref{eq:lem-start-recursive} for the case $w \le c_1$.
From Lemma~\ref{lem:start-w}, $w^{(1)} \le w + 2 \le c_1 + 2 = c_2$, which implies that
$(c_2 - w^{(1)})_+ = c_2 - w^{(1)}$.
Therefore,
\begin{align}
\phi - q \phi^{(0)} - (1 - q) \phi^{(1)} - h 
&= \phi - \phi^{(1)} - h + q(\phi^{(1)} - \phi^{(0)}) \nonumber\\
&\ge \phi - \phi^{(1)} - h \label{eq:lem-start-2}\\
&= w + \alpha_2 (c_2 - w)_+ - w^{(1)} - \alpha_2 (c_2 - w^{(1)})_+ - h \nonumber\\
&= w + \alpha_2 (c_2 - w) - w^{(1)} - \alpha_2 (c_2 - w^{(1)}) - h \nonumber \\
&= (w^{(1)} - w)(\alpha_2 - 1) - h \nonumber\\ 
&\ge h(\alpha_2 - 1) - h \label{eq:lem-start-3} \\ 
&\ge 0, \nonumber
\end{align}
where inequality~\eqref{eq:lem-start-2} follows from Lemma~\ref{lem:ineq-start-c3} and the fact that $w \le c_1 \le c_3$,
and inequality~\eqref{eq:lem-start-3} follows from Lemma~\ref{lem:start-w}.

Whenever $w \ge c_1$, the following inequality holds:
\begin{align} 
q
&\ge 1 - \frac{\alloc_{k,t}}{\amt_k} \nonumber\\
&\ge 1 - \frac{u + h\mulc u/w}{\amt_k} \nonumber\\
&\ge 1 - (1-p)(1 + h\mulc/w) \label{eq:start-1} \\
&\ge p - \mulc/w \nonumber\\
&\ge p/2, \label{eq:lem-start-bound-q}
\end{align}
where inequality~\eqref{eq:start-1} follows from the fact that
$u = \lbf_{k,t-1} < \gamma \le (1-p) \amt_k$.

Next, we prove \eqref{eq:lem-start-recursive} for the case $c_1 \le w \le c_3$.
Therefore
\begin{align}
\phi - q \phi^{(0)} - (1 - q) \phi^{(1)} - h 
&= \phi - \phi^{(1)} - h + q(\phi^{(1)} - \phi^{(0)}) \nonumber\\
&\ge \phi - \phi^{(1)} - h + q(\frac{6h}{p}) \label{eq:lem-start-4}\\
&\ge \phi - \phi^{(1)} - h + 3h \label{eq:lem-start-10}\\
&\ge w - w^{(1)} - h + 3h \nonumber\\
&\ge 0 \label{eq:lem-start-5}
\end{align}
where inequality~\eqref{eq:lem-start-4} follows from Lemma~\ref{lem:ineq-start-c3},
inequality~\eqref{eq:lem-start-10} follows from inequality~\eqref{eq:lem-start-bound-q},
and inequality~\eqref{eq:lem-start-5} follows from Lemma~\eqref{lem:start-w}.

Lastly, we prove inequality~\eqref{eq:lem-start-recursive} for $w \ge c_3$.
The bounds on $w^{(0)}$ and $w^{(1)}$, and Lemma~\ref{lem:simple-ineq}.\ref{itm:simple-ineq-2} imply that 
\begin{equation} \label{eq:lem-start-11}
w^{(1)} \ge w \ge w^{(0)} \ge w - (s/u-1)h \ge c_2.
\end{equation}
Thus,
\begin{align}
\phi - q \phi^{(0)} - (1 - q) \phi^{(1)} - h 
&\ge w - q w^{(0)} - (1-q) w^{(1)} - h \label{eq:lem-start-12} \\
&\ge w - \frac{p}{2} w^{(0)} - \left(1-\frac{p}{2}\right) w^{(1)} - h \label{eq:lem-start-7}\\
&\ge w - \frac{p}{2} \max\left(w/2, w - \frac{(s/u-1)h}{2(1 + \mulc /w)} \right) - \left(1-\frac{p}{2}\right) (w+2h) - h \label{eq:lem-start-8} \\
&\ge w - \frac{p}{2} \max\left(w/2, w - \frac{(s/u-1)h}{4} \right) - \left(1-\frac{p}{2}\right) w -2h - h \label{eq:lem-start-9} \\
&\ge \frac{p}{2} \min\left(w/2, \frac{(s/u-1)h}{4} \right) - 2h - h \nonumber \\
&\ge 0. \label{eq:lem-start-high-k}
\end{align}
where 
inequality~\eqref{eq:lem-start-12} follows from inequality~\eqref{eq:lem-start-11},
line~\eqref{eq:lem-start-7} follows from \eqref{eq:lem-start-bound-q},
line~\eqref{eq:lem-start-8} follows from Lemma~\ref{lem:start-w},
line~\eqref{eq:lem-start-9} follows from the fact that $w \ge c_1 \ge \mulc$,
and line~\eqref{eq:lem-start-high-k} follows from Lemma~\ref{lem:simple-ineq}.\ref{itm:simple-ineq-3}-\ref{itm:simple-ineq-4}.
\subsubsection{Proof of Lemma~\ref{lem:err-small}} \label{sec:At-err}

Fix an integer $k$, $1 \le k \le n$.
For any $t$, $1 \le t \le n$, let
\[ w_{k,t} = \exp\left( \frac{\saf_{k,t}}{\mulc \lbf_{k,t}} \right). \]
We will prove by induction on $m \ge 0$ that for any $t > \rvt_{k,1}$,
\begin{equation} \label{eq:small-main}
\E \left[ \sum_{i = t}^{\min(t+m-1, \rvt)} \frac{(\alloc_{k,i} - \lbf_{k,i-1})_+}{\add_{k,i}} \middle\vert Z_{k,t-1} \right]
\le \phi(w_{k,t-1}, \err_{k,t-1}),
\end{equation}
where
\[
\phi(w, \err) = \begin{cases}
c_2 w^2 \err + c_4\left( \frac{c_1}{\epsilon} - w \right)_+ + c_5 \left( \frac{2}{\eta} - \frac{1}{\epsilon} \right) &
\err \ge \frac{\eta}{2} \\
0 & \err < \frac{\eta}{2}.
\end{cases},
\]
and
\begin{align*}
c_2 &= 1\\
c_3 &= \frac{c^2 (1-p) + 7}{c(1-p) - 2} \\
c_1 &= c_3 + 2\\
c_4 &= 1 + c_2 (2 c_3 + 2) \\
c_5 &= c_4 c_1 (\log c_1 + 3).
\end{align*}

The proof is by induction on $m$. If $m = 0$ then inequality~\eqref{eq:small-main} holds since $\phi(w, \err) \ge 0$.
Assume therefore that $m > 0$. Fix some values of $t > \rvt_{k,1}$, and fix $Z_{t-1}$.
If $\err_{k,t-1} \le \eta$, then $\rvt < t$, and inequality~\eqref{eq:small-main} holds since $\phi(w, \err) \ge 0$.
Assume therefore that $\err_{k,t-1} > \eta$.
Denote $w = w_{k,t-1}$, $\err = \err_{k,t-1}$ and $u = \lbf_{k,t-1}$.
Let $w_0$ be the value that $w_{k,t}$ gets if $\scc_{k,t} = 0$, and let $w_1$ be its value if $\scc_{k,t} = 1$.
Similarly define $\err_0$, $\err_1$, $u_0$ and $u_1$.
Denote $h = \frac{(\alloc_{k,t} - \lbf_{k,t-1})_+}{\add_{k,t}}$.
Let $q = \Pr[\scc_{k,t} = 0 \mid Z_{t-1}]$.
To complete the proof, it is sufficient to prove that
\begin{align} \label{eq:small-rec}
\phi(w,\err) \ge h + (1 - q) \phi(w_1, \err_1) + q \phi(w_0, \err_0).
\end{align}
We can replace $\err_1$ with $\err$, since they are equal.

\begin{lemma} \label{lem:small-w}
	The following hold:
	\begin{enumerate}
		\item \label{itm:small-w-1}
		\[ w + h \le w_1 \le w + h + \frac{h^2}{k}. \]
		\item \label{itm:small-w-2}
		\[ w_0 = (w_1)^{\frac{1-\epsilon}{1-\epsilon_0}}. \]
		\item \label{itm:small-w-3}
		\[ w_0 \le w \le w_1. \]
	\end{enumerate}
\end{lemma}

\begin{proof}
	Start by proving item~\ref{itm:small-w-1}. It holds that
	\[
	w_1
	= \exp \left( \frac{\saf_{k,t}}{\mulc u_1} \right)
	= \exp \left( \frac{\saf_{k,t-1} + h \add_{k,t}}{\mulc \lbf_{k,t-1}} \right)
	= \exp \left( \frac{\saf_{k,t-1} + h \lbf_{k,t-1} \mulc / w_{t-1}}{\mulc \lbf_{k,t-1}} \right)
	= w e^{h/w}.
	\]
	Since $0 \le \frac{h}{w} \le 1$, applying the inequality $1 + x \le \exp(x) \le 1 + x + x^2$ which holds whenever $0 \le x \le 1$, suffices to complete the proof of item~\ref{itm:small-w-1}.
	
	We proceed to proving item~\ref{itm:small-w-2}. The value of $\saf_{k,t}$ is defined by $Z_{t-1}$, and does not depend on $\scc_{k,t}$. Therefore,
	\[
	w_0
	= \exp \left( \frac{\saf_{k,t}}{\mulc u_0} \right)
	= \exp \left( \frac{\saf_{k,t}}{\mulc u_1} \right)^{\frac{u_1}{u_0}}
	= \left(w_1\right)^{\frac{(1 - \err_1)\amt_k}{(1 - \err_0)\amt_k}},
	\]
	which completes the proof of item~\ref{itm:small-w-2}.
	
	Item~\ref{itm:small-w-3} is proved in Lemma~\ref{lem:start-w}.
\end{proof}

\begin{proposition} \label{prop:small-mon}
	The function $\phi(w, \err)$ is monotonic non-decreasing in $\err$.
\end{proposition}

\begin{proof}
	Follows immediately from the fact that $c_5 \ge c_1 c_4$.
\end{proof}

\begin{lemma} \label{lem:small-e0}
	It holds that
	\[
	\phi(w_0, \err_0) - \phi(w_1, \err_0)
	\le c_4 c_1 (\log c_1 + 2) \left( \frac{1}{\err_0} - \frac{1}{\err} \right).
	\]
\end{lemma}

\begin{proof}
	If $\err_0 > \eta/2$ then $\phi(w_0,\err_0) - \phi(w_1,\err_0)=0$.
	Otherwise, since $w_0 \leq w_1$,
	\begin{align}
	\phi(w_0,\err_0) - \phi(w_1,\err_0)
	&= c_2 \err_0 \left((w_0)^2-(w_1)^2 \right) + c_4\left((c_1/\err_0 - w_0)_+ - (c_1/\err_0 - w_1)_+\right) 		
	\nonumber\\
	&\leq c_4 \left((c_1/\err_0 - w_0)_+ - (c_1/\err_0 - w_1)_+\right). \label{eq:small-0}
	\end{align}
	
	We will show that 
	\begin{equation} \label{eq:small-4}
	\eqref{eq:small-0} \le c_4 (c_1/\err_0 - c_1/\err) + c_4 (c_1/\err - (c_1/\err)^{\frac{1-\err}{1-\err_0}}).
	\end{equation}
	If $w_0 \geq c_1/\err_0$ then this inequality holds since $\eqref{eq:small-0} = 0$.
	If $w_0 \le c_1/\err \le w_1$ then
	\begin{align}
	\eqref{eq:small-0}
	&= c_4 c_1/\err_0 - c_4 w_0 \nonumber\\
	&= c_4 (c_1/\err_0 - c_1/\err) + c_4 (c_1/\err - w_0) \nonumber \\
	&= c_4 (c_1/\err_0 - c_1/\err) + c_4 (c_1/\err - w_1^{\frac{1-\err}{1-\err_0}}) \label{eq:small-1} \\
	&\leq c_4 (c_1/\err_0 - c_1/\err) + c_4 (c_1/\err - (c_1/\err)^{\frac{1-\err}{1-\err_0}}), \nonumber
	\end{align}
	where inequality~\eqref{eq:small-1} follows from Lemma~\ref{lem:small-w}.
	
	If $w_1 \le c_1/\err$ then
	\begin{align}
	\eqref{eq:small-0}
	&= c_4(w_1 - w_0) \nonumber \\
	&= c_4 \left(w_1 - w_1^{\frac{1-\err}{1-\err_0}} \right) \label{eq:small-2} \\
	&\leq c_4\left(c_1/\err - (c_1/\err)^{\frac{1-\err}{1-\err_0}}\right), \label{eq:small-3}
	\end{align}
	where equality~\eqref{eq:small-2} follows from Lemma~\ref{lem:small-w},
	and inequality~\eqref{eq:small-3} follows from the fact that the function
	$x - x^\alpha$ is monotonic increasing in $x$, assuming a fixed $0 < \alpha < 1$.
	This completes the proof of inequality~\eqref{eq:small-4}.
	
	To conclude the proof, it is sufficient to show that
	\begin{equation} \label{eq:small-5}
	(c_1/\err - (c_1/\err)^{\frac{1-\err}{1-\err_0}}) 
	\leq (1/\err_0 - 1/\err) c_1 (\log(c_1) +1).
	\end{equation}
	Let $\gamma = \err/\err_0-1$.
	Bounding
	\begin{align}
	1 - (c_1/\err)^{(1-\err)/(1-\err_0)-1}
	&= 1 - (c_1/\err)^{(\err_0-\err)/(1-\err_0)}\nonumber \\
	&= 1 - e^{\log(c_1/\err)(\err_0-\err)/(1-\err_0)}\nonumber \\
	&= 1 - e^{-\log(c_1/\err)\gamma \err_0/(1-\err_0)}\nonumber \\
	&\le 1 - e^{-\log(c_1/\err_0)\gamma \err_0/(1-\err_0)}\nonumber \\
	&= 1 - e^{-\log(c_1) \gamma \err_0/(1-\err_0) - \log(1/\err_0) \gamma \err_0/(1-\err_0)} \nonumber\\
	&\leq 1 - e^{-\log(c_1) \gamma - \frac{1}{2} \gamma /(1-\err_0)}\label{eq:small-6}\\
	&\leq 1 - e^{-(\log(c_1) +1)\gamma} \label{eq:small-7} \\
	&\leq (\log(c_1) +1)\gamma, \label{eq:small-8}
	\end{align}
	where inequality~\eqref{eq:small-6} follows from the fact that $\epsilon_0 \le p \le \frac{1}{2}$, therefore $\frac{\epsilon_0}{1-\epsilon_0} \le 1$, and from the fact that $\epsilon_0 \log \frac{1}{\epsilon_0} \le \frac{1}{2}$, for any $\epsilon_0 > 0$;
	inequality~\eqref{eq:small-7} follows from the fact that $\epsilon_0 \le p \le \frac{1}{2}$;
	and inequality~\eqref{eq:small-8} follows from the inequality $e^{-x} \geq 1-x$ which holds for all $x \in \mathbb{R}$.
	
	Thus, we conclude the proof of inequality~\eqref{eq:small-5} and the proof of this lemma:
	\begin{align}
	c_1/\err - (c_1/\err)^{\frac{1-\err}{1-\err_0}} 
	&= \frac{c_1}{\err} \left(1 - \left( \frac{c_1}{\err} \right)^{\frac{1-\err}{1-\err_0}-1}\right) \nonumber \\
	&\leq \frac {c_1}{\err} (\log(c_1) +1)\gamma \label{eq:small-9} \\
	&= (1/\err_0 - 1/\err) c_1 (\log(c_1) +1). \nonumber
	\end{align}
	where inequality~\eqref{eq:small-9} follows from inequality~\eqref{eq:small-8}.
\end{proof}

We start by proving \eqref{eq:small-rec}, assuming that $w \err \leq c_3$.
Let $\Delta = w_1-w$.
Lemma~\ref{lem:small-w} states that $h \leq \Delta \leq 2h$. Therefore, 
\begin{equation} \label{eq:small-14}
w_1\err \leq c_3 + 2 = c_1.
\end{equation}
This implies that
\begin{align}
\phi(w,\err) - \phi(w_1,\err)
&= c_2(w^2-w_1^2) \err + c_4(c_1/\err-w) - c_4(c_1/\err - w_1) \nonumber\\
&= c_2\err (-2w \Delta - \Delta^2) + c_4 \Delta \nonumber \\
&\geq  -c_2 (2 c_3 \Delta + 2\Delta) + c_4 \Delta \label{eq:phi-kkp-epsk}\\
&=\Delta (-c_2 (2 c_3 + 2) + c_4) \nonumber\\
&= \Delta \label{eq:use-def-c4} \\
&\geq h, \label{eq:small-13}
\end{align}
where \eqref{eq:phi-kkp-epsk} follows from the assumption $\err w \leq c_3$ and the inequality $\Delta \le 2$; and \eqref{eq:use-def-c4} follows from the definition of $c_4$.

If $\err_0 < \eta/2$, then
\begin{align}
&\phi(w,\err) - h - (1-q) \phi(w_1,\err) - q \phi(w_0,\err_0) \nonumber\\
&= (\phi(w,\err) - \phi(w_1,\err)) - h - q (\phi(w_1,\err_0) - \phi(w_1,\err)) - q(\phi(w_0,\err_0) - \phi(w_1,\err_0)) \nonumber\\
&\ge h - h - q (\phi(w_1,\err_0) - \phi(w_1,\err)) - q(\phi(w_0,\err_0) - \phi(w_1,\err_0)) \label{eq:small-10}\\
&\ge - q(\phi(w_0,\err_0) - \phi(w_1,\err_0)) \label{eq:small-11}\\
&= 0, \label{eq:small-12}
\end{align}
where inequality~\eqref{eq:small-10} follows from inequality~\eqref{eq:small-13},
inequality~\eqref{eq:small-11} follows from Proposition~\ref{prop:small-mon},
and inequality~\eqref{eq:small-10} follows from the fact that $\err_0 < \eta/2$.

If $\err_0 \geq \eta/2$, then
\begin{align}
\phi(w_1,\err_0) - \phi(w_1,\err)
&\leq c_4[c_1/\err_0 - w_1]_+ + c_5[2/\eta - 1/\err_0]_+ - c_4[c_1/\err_0 - w_1]_+ - c_5[2/\eta - 1/\err_0]_+ \nonumber \\
&= c_4(c_1/\err_0 - w_1) + c_5(2/\eta - 1/\err_0) - c_4(c_1/\err - w_1) - c_5(2/\eta - 1/\err)
\label{eq:kp-epsp-epspp} \\
&= (c_4c_1- c_5)(1/\err_0 - 1/\err) \label{eq:small-16}
\end{align}
where inequality~\eqref{eq:kp-epsp-epspp} follows from inequality~\eqref{eq:small-14} and the fact that $\err_0 \ge \frac{\eta}{2}$. Thus,
\begin{align}
&\phi(w,\err) - h - (1-q) \phi(w_1,\err) - q \phi(w_0,\err_0) \nonumber\\
&= (\phi(w,\err) - \phi(w_1,\err)) - h - q (\phi(w_1,\err_0) - \phi(w_1,\err)) - q(\phi(w_0,\err_0) - \phi(w_1,\err_0)) \label{eq:small-15}\\
&\geq h - h + (c_5 - c_4c_1)(1/\err_0 - 1/\err) - (1/\err_0 - 1/\err) c_1 c_4(\log(c_1) +2) \nonumber\\
&= (c_5 - c_4 c_1 (\log(c_1) +3))(1/\err_0 - 1/\err) \nonumber\\
&= 0, \nonumber
\end{align}
where inequality~\eqref{eq:small-15} follows from inequalities \eqref{eq:small-13} and \eqref{eq:small-16}, and Lemma~\ref{lem:small-e0}.
This concludes the proof of inequality~\eqref{eq:small-rec} for the case $\err w \le c_3$.

Next, assume that $w \err \geq c_3$.
Therefore,
\begin{align}
\phi(w_1,\err) - \phi(w,\err)
&\leq c_2 w_1^2\err - c_2 w^2 \err \nonumber\\
&= c_2 (2 w \Delta + \Delta^2) \err \nonumber\\
&\leq c_2 (2 w (h + h^2/w) + (2h)^2) \err \nonumber\\
&= c_2 (2wh + 6h) \err, \label{eq:small-17}
\end{align}
using Lemma~\ref{lem:small-w}.

Since $\err w \ge c_3 > \mulc$, it holds that
\begin{equation} \label{eq:small-24}
\amt_k - \alloc_{k,t}
= \amt_k - u - \frac{h c u}{w}
= \amt_k \left( \err - \frac{h c u}{w\amt_k}  \right)
\ge \amt_k \err \left( 1 - \frac{c}{w\err}  \right)
> 0.
\end{equation}
Therefore, $u_0 = \alloc_{k,t} = u + \frac{h c u}{w}$. This implies that
\begin{equation} \label{eq:small-18}
\err_0 
= 1- \frac{u_0}{\amt_k} 
= 1 - \frac{u + (c h u)/w}{\amt_k}
= \err - \frac{c h u}{\amt_k w}.
\end{equation}
Since $t > \rvt_{k,1}$, it holds that $u \ge (1-p) \amt_k$, therefore,
inequality~\eqref{eq:small-18} implies that
\begin{equation} \label{eq:small-23}
\err_0
= \err - \frac{c h u}{\amt_k w}
\le \err -  \frac{h (1-p) c}{w}.
\end{equation}
Additionally,
\begin{equation} \label{eq:small-25}
\err_0
= \err - \frac{c h u}{\amt_k w}
\ge \err - \frac{c h}{w}
= \err \left(1 - \frac{c h}{w\err} \right)
\ge \err \left(1 - \frac{c}{c_3} \right)
\ge \frac{\err}{2}
> \frac{\eta}{2}.
\end{equation}
This implies that
\begin{align}
\phi(w_1,\err) - \phi(w_1,\err_0)
&= c_2 w_1^2(\err - \err_0) + c_4((c_1/\err-w_1)_+ - (c_1/\err_0-w_1)_+) + c_5(1/\err_0 - 1/\err) \nonumber\\
&\geq c_2 w^2(\err - \err_0) + c_4(c_1/\err- c_1/\err_0) + c_5(1/\err_0 - 1/\err) \nonumber\\
&\geq c_2 w h c(1-p) + (c_5 - c_4 c_1)(1/\err_0 - 1/\err). \label{eq:small-19}
\end{align}
where inequality~\eqref{eq:small-19} follows from inequality~\eqref{eq:small-23}.

Additionally, inequality~\eqref{eq:small-24} and inequality~\eqref{eq:small-25} imply that
\begin{equation} \label{eq:small-21}
q
= \err_0
\ge \err\left(1 - \frac{c}{c_3}\right).
\end{equation}
To complete the proof:
\begin{align}
&\phi(w,\err) - h - (1-q) \phi(w_1,\err) - q \phi(w_0,\err_0) \nonumber\\
&= (\phi(w,\err) - \phi(w_1,\err)) - h + q (\phi(w_1,\err) - \phi(w_1,\err_0)) + q(\phi(w_1,\err_0) - \phi(w_0,\err_0)) \nonumber\\
&\ge -c_2 (2 w h + 6 h) \err 
- h
+ q\left(c_2 w h c (1-p) + (c_5 - c_4 c_1) \left( \frac{1}{\err_0} - \frac{1}{\err} \right)\right)
- q  c_4 c_1 (\log c_1 + 2) \left(\frac{1}{\err_0} - \frac{1}{\err} \right) \label{eq:small-20}\\
&= -c_2 (2 w h + 6 h) \err - h + q\left(c_2 w h c (1-p)\right) \nonumber \\
&\ge -c_2 (2 w h + 6 h) \err - h + \err \left( 1 - \frac{c}{c_3} \right)\left(c_2 w h c (1-p)\right) \label{eq:small-22}\\
&= w \err h \left( -2 - \frac{6}{w} - \frac{1}{w \err} + \left(1 - \frac{c}{c_3}\right) c (1-p) \right) \nonumber \\
&\ge w \err h \left( -2 - \frac{7}{w \err} + c (1-p) - \frac{c^2 (1-p)}{c_3}\right) \nonumber \\
&\ge w \err h \left( c (1-p) -2 - \frac{1}{c_3} (7 + c^2 (1-p))\right) \nonumber \\
&= 0, \nonumber
\end{align}
where inequality~\eqref{eq:small-20} follows from inequality~\eqref{eq:small-17},
inequality~\eqref{eq:small-19}, and Lemma~\ref{lem:small-e0};
and inequality~\eqref{eq:small-22} follows from inequality~\eqref{eq:small-21}.
\subsubsection{Proof of Lemma~\ref{lem:k-sqr}} \label{sec:At-sqr}
Fix some integer $k$, $1 \le k \le \nban$. 
Define the values
\[ c' = \frac{2}{\mulc (1-p) \amt_k}, \]
\[ \gamma = \log \left(\frac{1 - p/4}{1 - p/2}\right) \frac{1-p}{2}, \]
\[ s_0 = \max \left( \mulc \amt_k \log \left( 2\mulc / p \right), \amt_k \mulc \ln \frac 1 \gamma \right), \]
\[ c'' = 4c'/p. \]
Let $\rvt$ be the first iteration $t$ such that $\saf_{k,t} \ge s_0$.

\begin{proposition} \label{prop:alpha-s0}
	Fix $t > \rv{\tau}$, let $x = (\alloc_{k,t} - \lbf_{k,t-1})_+$. It holds that:
	\begin{enumerate}
		\item \label{itm:alpha-s0-1}
		\[
		x \le \frac{p \amt_k}{2}.
		\]
		\item \label{itm:alpha-s0-2}
		\[
		\exp(c' x) \le (1-p/4)/(1-p/2).
		\]
	\end{enumerate}
\end{proposition}

\begin{proof}
	First, we prove item~\ref{itm:alpha-s0-1}:
	\[ 
	x 
	\le \add_{k,t} 
	= \mulc \lbf_{k,t-1} \exp(-s_{k,t-1}/(\lbf_{k,t-1} \mulc))
	\le \mulc \amt_k \exp(-s_0/(\amt_k \mulc))
	= \mulc \amt_k \exp(-\log (2 \mulc/p))
	= \frac{\amt_k p}{2}.
	\]
	
	Next, we prove item~\ref{itm:alpha-s0-2}.
	Start by bounding $x$:
	\[
	x
	\le \mulc \amt_k \exp(-s_0/(\amt_k \mulc))
	\le \mulc \amt_k \exp(\log \gamma)
	= \mulc \amt_k \gamma.
	\]
	To complete the proof, we estimate
	\[
	\exp(c' x)
	\le \exp \left( \frac{2}{\mulc (1-p) \amt_k} \mulc \amt_k \log \left(\frac{1 - p/4}{1 - p/2}\right) \frac{1-p}{2} \right)
	= \frac{1 - p/4}{1 - p/2}.
	\]
\end{proof}

We prove by induction on $m \ge 0$, that for any $\rv{\tau} \le t \le \rv{\tau_{k,1}}$ it holds that
\begin{equation} \label{eq:k-sqr-induction}
\E \left[ \exp\left(c' \left(\saf_{k,\min(t+m, \rv{\tau_{k,1}})} - \saf_{k,t} \right)\right) \middle\vert \lbf_{k,t} \right]
\le \exp\left( c''(\amt_k - \lbf_{k,t}) \right).
\end{equation}

The base of induction is clear: whenever $m = 0$ the left-hand side equals 1, and the right-hand side is at least 1.
If $t = \rv{\tau_{k,1}}$ then, from the same reason the inequality holds.

For the induction step, assume that $m > 0$, and take some $\rv{\tau} \le t < \rv{\tau_{k,1}}$.
Proposition~\ref{prop:alpha-s0} implies that
\[
\alloc_{k,t+1} 
\le (\alloc_{k,t+1} - \lbf_{k,t})_+ + \lbf_{k,t}
\le \frac{p \amt_k}{2} + \amt_k(1 - p)
= (1 - p/2) \amt_k
\]
Therefore,
\[
\Pr[\scc_{k,t+1} = 0]
= 1 - \mo \left(\frac{\alloc_{k,t+1}}{\amt_k} \right)
\ge 1 - (1 - p/2)
= p / 2.
\]

Let
\[
x = (\alloc_{k,t+1} - \lbf_{k,t})_+.
\]
It holds that
\[
\lbf_{k,t+1} = \begin{cases}
\lbf_{k,t+1} + x & \scc_{k,t+1} = 0 \\
\lbf_{k,t} & \scc_{k,t} = 1
\end{cases}.
\]
Therefore, by induction hypothesis, it holds that
\begin{align*}
&\E \left[ \exp\left(c' \left(\saf_{k,\min(t+m, \rv{\tau_{k,1}})} - \saf_{k,t} \right)\right) \middle\vert \lbf_{k,t}, x \right] \\
&= \sum_{b=0}^1 \Pr[\scc_{k,t+1}=b] 
\E \left[ \exp\left(c' \left(\saf_{k,\min(t+m, \rv{\tau_{k,1}})} - \saf_{k,t} \right)\right) \middle\vert 
\lbf_{k,t}, x, \scc_{k,t+1}=b \right]\\
&= \sum_{b=0}^1 \Pr[\scc_{k,t+1}=b] 
\E \left[ \exp(c'x) \exp\left(c' \left(\saf_{k,\min(t+m, \rv{\tau_{k,1}})} - \saf_{k,t+1} \right)\right) \middle\vert 
\lbf_{k,t}, x, \scc_{k,t+1}=b \right]\\
&\le \sum_{b=0}^1 \Pr[\scc_{k,t+1}=b] 
\E \left[ \exp(c'x) \exp\left(c''(\amt_k - \lbf_{k,t+1})\right) \middle\vert 
\lbf_{k,t}, x, \scc_{k,t+1}=b \right]\\
&= \Pr[\scc_{k,t+1}=0] \exp(c'x) \exp\left(c''(\amt_k - \lbf_{k,t} - x)\right)
+\Pr[\scc_{k,t+1}=1] \exp(c'x) \exp\left(c''(\amt_k - \lbf_{k,t})\right)\\
&\le \frac{p}{2} \exp\left(c''(\amt_k - \lbf_{k,t} - x)+c'x \right)
+(1-\frac{p}{2}) \exp\left(c''(\amt_k - \lbf_{k,t}) + c'x\right).
\end{align*}

We would like to show that
\[
\frac{p}{2} \exp\left(c''(\amt_k - \lbf_{k,t} - x)+c'x \right)
+(1-\frac{p}{2}) \exp\left(c''(\amt_k - \lbf_{k,t}) + c'x\right)
\le \exp\left( c''(\amt_k - \lbf_{k,t}) \right),
\]
which is equivalent to showing that the function
\begin{equation} \label{eq:k-sqr-topr}
\phi(y) = \frac{p}{2} + (1-\frac{p}{2}) \exp\left(c'' y \right) - \exp\left( (c''-c')y \right),
\end{equation}
satisfies $\phi(x)\le 0$.
It trivially holds that $\phi(0) = 0$,
and we will show that $\frac{d\phi}{dy}(y) \le 0$ for all $0 \le y \le x$.
This will imply that $\phi(x) \le 0$.
Indeed,
\begin{align}
\frac{d\phi}{dy}(y)
&= (1-p/2)c'' \exp(y c'') - (c''-c') \exp((c''-c')y) \nonumber\\
&= \exp(c'' y) \left((1-p/2)c'' - (c''-c') \exp(-c'y) \right) \nonumber\\
&\le \exp(c'' y) \left((1-p/2)c'' - (c''-c') \exp(-c'x) \right) \label{eq:k-sqr-0} \\
&\le \exp(c'' y) \left((1-p/2)c'' - (c''-c') \frac{1-p/2}{1-p/4} \right) \nonumber\\
&= 0 \nonumber,
\end{align}
where inequality~\eqref{eq:k-sqr-0} follows from Proposition~\ref{prop:alpha-s0}.\ref{itm:alpha-s0-2}.

This proves that 
\[
\E \left[ \exp\left(c' \left(\saf_{k,\min(t+m, \rv{\tau_{k,1}})} - \saf_{k,t} \right)\right) \middle\vert \lbf_{k,t}, x \right]
\le \exp\left( c''(\amt_k - \lbf_{k,t}) \right),
\]
and this inequality holds for every possible value of $x$,
therefore the proof of inequality~\eqref{eq:k-sqr-induction} is concluded.

To conclude the proof, note that
\[
\saf_{k,\rv\tau} 
= \saf_{k, \rvt-1} + (\alloc_{k,\rvt}-\lbf_{k,\tau-1})_+
\le s_0 + \add_{k,\rvt}
\le s_0 + c \amt_k
\le c \amt_k(\log \frac{4c}{\gamma p}).
\]
Thus,
\begin{align*}
\E \left[ \exp(c' \saf_{k,\rv{\tau_1}}) \right]
&= \E \left[ \exp(c' \saf_{k,\rvt}) \exp(c' \saf_{k,\rvt_1} - \saf_{k,\rvt}) \right]\\
&\le \exp(c' c \amt_k(\log \frac{4c}{\gamma p})) \E \left[ \exp(c' \saf_{k,\rvt_1} - \saf_{k,\rvt}) \right]\\
&\le \exp\left(\frac{2}{(1-p)}\right) \frac{4c}{\gamma p} \exp(c'' \amt_k)\\
&\le \exp\left(\frac{2}{(1-p)}\right) \frac{4c}{\gamma p} \exp\left( \frac{8}{cp(1-p)} \right).
\end{align*}
\subsubsection{Concluding the proof} \label{sec:At-conc}
Fix an arm $k$, $1 \le k \le n$. It holds that:
\begin{align}
\sum_{t \in T \colon k \in A_t} (1 - \mo(\alloc_{k,t} / \amt_k))
&\le \sum_{t > \rvt_{k,0} \colon k \in A_t} (1 - \mo(\alloc_{k,t} / \amt_k)) \nonumber \\
&= \sum_{t \colon \rvt_{k,0} < t \le \rvt_{k,1},\ k \in A_t} (1 - \mo(\alloc_{k,t} / \amt_k)) \label{eq:At-0} \\
&+ \sum_{t \colon t > \rvt_{k,1},\ k \in A_t} (1 - \mo(\alloc_{k,t} / \amt_k)). \label{eq:At-1}
\end{align}
We will start by bounding the amount in the equation line marked~\eqref{eq:At-0} and proceed in bounding the amount in \eqref{eq:At-1}.

For any iteration $t$ where $k \in A_t$, $\frac{(\alloc_{k,t} - \lbf_{k,t-1})_+}{\add_{k,t}} = 1$.
Therefore,
\begin{align}
\sum_{\substack{\rvt_{k,0} < t \le \rvt_{k,1} \colon \\ k \in A_t}} (1 - \mo(\alloc_{k,t} / \amt_k))
&\le \sum_{\substack{\rvt_{k,0} < t \le \rvt_{k,1} \colon\\ k \in A_t}} 1 \nonumber \\
&= \sum_{\substack{\rvt_{k,0} < t \le \rvt_{k,1} \colon\\ k \in A_t}} \frac{(\alloc_{k,t}- \lbf_{k,t-1})_+}{\add_{k,t}} \nonumber \\
&\le C_1 \log K + C_1', \label{eq:At-2}
\end{align}
for some constants $C_1, C_1' > 0$ depending only on $\mulc$, where the last inequality follows from Lemma~\ref{lem:aux-At-add}.

We proceed by bounding the amount in \eqref{eq:At-1}.
Denote $\err_{k,t} = \frac{\amt_k - \lbf_{k,t}}{\amt_k}$.
For any value of $0 < \eta < 1$, let $\rvt'_{k,\eta}$ be the first iteration $t$ that $\err_{k,t} \le \eta$.
Lemma~\ref{lem:err-small} implies that there is a constant, $C_2 > 0$, depending only on $\mulc$,
such that
\[
\E\left[ \sum_{t = \rvt_{k,1} + 1}^{\rvt'_{k,\eta}} \frac{(\alloc_{k,t} - \lbf_{k,t-1})_+}{\add_{k,t}}
\middle\vert \rvt_{k,1}, Z'_{\rvt_{k,1}} \right]
\le C_2 \left(\exp\left( \frac{2\saf_{k,\rvt_{k,1}} }{\mulc \amt_k (1-p)} \right) + \frac{1}{\eta} \right).
\]
Lemma~\ref{lem:k-sqr} bounds the expected value of 
$\E\left[\exp\left( \frac{2\saf_{k,\rvt_{k,1}} }{\mulc \amt_k (1-p)} \right)\right]$
by another constant, $C_3 > 0$, depending only on $\mulc$.
Combining these two results, we get that
\begin{equation} \label{eq:At-3}
\E\left[ \sum_{t = \rvt_{k,1} + 1}^{\rvt'_{k,\eta}} \frac{(\alloc_{k,t} - \lbf_{k,t-1})_+}{\add_{k,t}} \right]
\le \frac{C_4}{\eta}, 
\end{equation}
for some constant $C_4 > 0$ depending only on $\mulc$.

Let $w_{k,\eta}$ be the number of iterations $t$, $\rvt_{k,1} < t \le \rvt'_{k, \eta}$, for which $k \in A_t$.
Equation~\eqref{eq:At-3} implies that
\begin{align}
\E w_{k,\eta}
&= \E \sum_{t \colon \rvt_{k,1} < t \le \rvt'_{k, \eta},\ k \in A_t} 1 \nonumber \\
&= \E \sum_{t \colon \rvt_{k,1} < t \le \rvt'_{k, \eta},\ k \in A_t} 
\frac{(\alloc_{k,t} - \lbf_{k,t-1})_+}{\add_{k,t}} \nonumber \\
&\le \E \sum_{t = \rvt_{k,1} + 1}^{\rvt'_{k,\eta}} 
\frac{(\alloc_{k,t} - \lbf_{k,t-1})_+}{\add_{k,t}} \nonumber \\
&\le \frac{C_4}{\eta}.
\label{eq:At-4}
\end{align}
Therefore,
\begin{align} 
\sum_{t \ge \rvt_{k,1}+1 \colon k \in A_t} (1 - \mo(\alloc_{k,t}/\amt_k))
&\le \sum_{t \ge \rvt_{k,1}+1 \colon k \in A_t} (1 - \mo(\lbf_{k,t-1}/\amt_k))\nonumber\\
&= \sum_{t \ge \rvt_{k,1}+1 \colon k \in A_t} \err_{k,t-1}\nonumber\\
&= \sum_{m=1}^\nit \sum_{\substack{
		t \ge \rvt_{k,1}+1 \colon\\ 
		k \in A_t \\ 
		\frac{1}{m+1} < \err_{k,t-1} \le \frac{1}{m}}
} \err_{k,t-1}
+ \sum_{\substack{t \ge \rvt_{k,1}+1 \colon\\ k \in A_t \\ \err_{k,t-1} \le \frac{1}{n+1}}} \err_{k,t-1} \nonumber\\
&\le \sum_{m=1}^\nit \sum_{
	\substack{
		t \ge \rvt_{k,1}+1 \colon\\ 
		k \in A_t \\ 
		\frac{1}{m+1} < \err_{k,t-1} \le \frac{1}{m}}
} \frac{1}{m}
+ \sum_{\substack{t \ge \rvt_{k,1}+1 \colon\\ k \in A_t \\ \err_{k,t-1} \le \frac{1}{n+1}}} \frac{1}{n+1} \nonumber\\
&\le \sum_{m=1}^\nit \left| \left\{ t \colon t > \rvt_{k,1},\ \rvt'_{k, 1/m} < t \le \rvt'_{k,1/(m+1)},\ k \in A_t \right\} \right|
\frac{1}{m}
+ \frac{n}{n+1} \nonumber\\
&\le \sum_{m=1}^\nit (w_{k,1/(m+1)} - w_{k,1/m})
\frac{1}{m} + 1 \nonumber\\
&\le -w_{k,1} + \sum_{m=2}^{n-1} w_{k,1/m} \left( \frac{1}{m-1} - \frac{1}{m} \right) + \frac{1}{n-1} w_{k,n} + 1 \nonumber\\
&\le 2 \sum_{m=2}^{n-1}  \frac{w_{k,1/m}}{m^2} + 3. \label{eq:At-5}
\end{align}
Inequality~\eqref{eq:At-4} implies that
\[
\E \left[ \sum_{m=2}^{n-1} w_{k,1/m} \frac{1}{m^2} \right]
\le C_4\sum_{m=2}^{n-1}  \frac{m}{m^2}
\le C_4(\log n + 4).
\]


\subsection{Proof of Lemma~\ref{lem:add-bound}} \label{sec:p-add}

Define $T' = T \cap \{ \rvt_{k,0}+1, \cdots, \rvt_{k,1} \}$, and $T'' = T \cap \{ \rvt_{k,1} + 1, \cdots, n \}$.
We will divide the sum that we have to bound into two summands: one over $T'$ and one over $T''$.

Start with $T'$.
\begin{align}
&\E \left[\sum_{t \in T' \colon k \in A_t} \frac{\add_{k,t}}{\lbf_{k,t-1}} 
+ \sum_{t \in T' \colon k \in B_t} \frac{r'_t}{\lbf_{k,t-1}} \right] \nonumber \\
&= \E \left[\sum_{t \in T' \colon k \in A_t} \frac{\add_{k,t}}{\lbf_{k,t-1}}
+ 2 \sum_{t \in T' \colon k \in B_t} \frac{(\alloc_{k,t} - \lbf_{k,t-1})_+}{\lbf_{k,t-1}} \right]
\label{eq:add-0} \\
&\le 2 \E \left[ \sum_{t = \rvt_{k,0}+1}^{\rvt_{k,1}} \frac{(\alloc_{k,t} - \lbf_{k,t-1})_+}{\lbf_{k,t-1}} \right] \nonumber \\
&\le 2 \E \left[ \sum_{t = \rvt_{k,0}+1}^{\rvt_{k,1}} \frac{(\alloc_{k,t} - \lbf_{k,t-1})_+}{\add_{k,t}} \right] \nonumber \\
&\le C_1 \log K + C_1'. \label{eq:add-1}
\end{align}
for some constants $C_1,C_1'>0$, depending only on $\mulc$. 
Inequality~\eqref{eq:add-0} follows from the fact that conditioned on $k \in B_t$, $k$ is allocated according to case B, hence $\alloc_{k,t}$ equals either $\lbf_{k,t-1}$ or $\lbf_{k,t-1} + \add'_t$, each with probability $1/2$;
inequality~\eqref{eq:add-1} follows from Lemma~\ref{lem:aux-At-add}.

Next, bound the sum that relates to $T''$.
Similarly to the calculation in Equality~\eqref{eq:add-0}:
\begin{align}
\E\left[ \sum_{t \in T'' \colon k \in A_t} \frac{\add_{k,t}}{\lbf_{k,t-1}} + \sum_{t \in T'' \colon k \in B_t} \frac{r'_t}{\lbf_{k,t-1}} \right] \nonumber \\ 
\le 2 \E\left[  \sum_{t \in T''} \frac{(\alloc_{k,t} - \lbf_{k,t-1})_+}{\lbf_{k,t-1}} \right] \nonumber\\
\le 2 \E\left[  \sum_{t \in T''} \frac{(\alloc_{k,t} - \lbf_{k,t-1})_+}{(1-p)\amt_k} \right] \nonumber\\
\le \frac{2}{(1-p)\amt_k} \E\left[ \sum_{1 \le t \le n \colon \lbf_{k,t-1}>0} (\alloc_{k,t} - \lbf_{k,t-1})_+ \right] \nonumber\\
= \frac{2}{(1-p) \amt_k} \E\left[ \saf_{k,n} \right]. \nonumber
\end{align}

To conclude the proof, we prove by induction on $t$, $1 \le t \le n$, that $\saf_{k,t} \le \mulc \amt_{k} H(t-1)$, where $H(t) = \sum_{i=1}^t \frac{1}{i}$ is the harmonic sum.
Trivially $\saf_{k,1} = 0 = H(0)$. Assume that this statement holds for $t$ and prove for $t+1$.
\begin{align}
\saf_{k,t+1}
&\le \saf_{k,t} + \mulc \lbf_{k,t} \add_{k,t+1} \nonumber \\
&= \saf_{k,t} + \mulc \lbf_{k,t} \exp \left( - \frac{\saf_{k,t}}{\mulc \lbf_{k,t}} \right) \nonumber \\
&\le \saf_{k,t} + \mulc \amt_k \exp \left( - \frac{\saf_{k,t}}{\mulc \amt_k} \right) \nonumber \\
&\le \mulc \amt_k H(t-1) + \mulc \amt_k \exp \left( - \frac{\mulc \amt_k H(t-1)}{\mulc \amt_k} \right) \label{eq:add-2} \\
&\le \mulc \amt_kH(t-1) + \mulc \amt_k e^{-\log t} \label{eq:add-3} \\
&= \mulc \amt_k H(t). \nonumber
\end{align}
where Inequality~\eqref{eq:add-2} follows from induction hypothesis, and from the fact that the function $x + \alpha \exp\left(-\frac{x}{\alpha} \right)$ is monotonic non-decreasing in $x$, for $x\ge 0$ and $\alpha > 0$, and Inequality~\eqref{eq:add-3} follows from the fact that $H(t-1) \ge \log t$, for all $t \ge 1$.

\subsection{Proof of Lemma~\ref{lem:notin-T}} \label{sec:p-notin-T}
We use the following variant of Azuma's inequality.

\begin{lemma} \label{lem:mart-sum}
	Let $Y_1, Y_2, \dots$ be an infinite sequence of random random variables, and let $X_1, X_2, \dots$ be random variables getting values from $\{0,1\}$.
	Assume that $X_i$ is a function of $Y_1, \dots, Y_i$ for all $i \ge 1$.
	For any $i \ge 1$, let $P_i$ be a random variable which is a function of $Y_1, \dots, Y_{i-1}$
	and equals $\Pr[X_i = 1 \mid Y_1, \dots, Y_{i-1}]$.
	The following statements hold:
	\begin{enumerate}
		\item \label{itm:mart1}
		Fix a number $r$, and let $\rvt_r$ be the random variable denoting the last number $i$ such that $\sum_{j=1}^i P_j \le r$.
		Assume that there exists some constant $m$ such that it always holds that $\rvt_r \le m$.
		Then, for any $0 < \delta \le 1$,
		\[
			\Pr\left[\sum_{i=1}^{\rvt_r} X_i > (1 + \delta)r \right] \le \exp\left(-\frac{\delta^2 r}{3} \right).
		\]
		\item \label{itm:mart2}
		Fix a number $r$, and let $\rvt_r$ be the random variable denoting the first number $i$ such that $\sum_{j=1}^i P_j > r$.
		Assume that there exists some constant $m$ such that it always holds that $\rvt_r \le m$.
		Then, for any $0 < \delta \le 1$,
		\[
			\Pr\left[\sum_{i=1}^{\rvt_r} X_i < (1 - \delta)r \right] \le \exp\left(-\frac{\delta^2 r}{2} \right).
		\]
	\end{enumerate}
\end{lemma}

This is a martingale version of the following bound on the relative error of independent random variables by \cite{CHE52}.

\begin{lemma} \label{lem:chernoff}
	Let $X_1, \dots, X_n$ be independent random variable getting values from $\{ 0,1 \}$.
	Let $X = \sum_{i=1}^n X_i$.
	Then, for all $0 < \delta < 1$,
	\begin{enumerate}
		\item 
			\[
				\Pr[X \ge (1 + \delta) \E X]
				\le \exp\left( - \frac{\delta^2 \E X}{3} \right).
			\]
		\item
			\[
				\Pr[X \le (1 - \delta) \E X]
				\le \exp\left( - \frac{\delta^2 \E X}{2} \right).
			\]
	\end{enumerate}
\end{lemma}

First note that we can assume in Lemma~\ref{lem:mart-sum}.\ref{itm:mart1}
that $\sum_{j=1}^{\rvt_r} P_j = r$.
Then, the proof is almost identical to the proof of Lemma~\ref{lem:mart-sum},
inductively bounding $\E \exp \left( t \sum_{i=1}^{\rvt_r} X_i \right) \le \exp\left( r (e^t -1) \right)$.
Lemma~\ref{lem:mart-sum}.\ref{itm:mart2} is proved similarly.

Before proving the following lemmas,
extend the values of $\alloc_{k,t}$, $\scc_{k,t}$, $\sas_{k,t}$ and $\xs_{k,t}$ for $t > \nit$,
by defining, for all $t > \nit$,
\[
\alloc_{k,t} = \min(\amt_k, 1),
\]
\[
\scc_{k,t} = \begin{cases}
1 & \text{with probability } \alloc_{k,t} / \amt_k \\
0 & \text{with probability } 1 - \alloc_{k,t} / \amt_k
\end{cases},
\]
\begin{equation} \label{eq:gslow-sas}
\sas_{k,t} = \sas_{k,t-1} + \alloc_{k,t},
\end{equation}
and
\[
\xs_{k,t} = \xs_{k,t-1} + \scc_{k,t}.
\]

Here is an auxiliary lemma.

\begin{lemma} \label{lem:bad-slow-bd}
	Fix some arm $k$, $1 \le k \le \nban$.
	Then,
	\[
	\E\left[\left| \left\{ 0 \le t \le \nit \colon \lbs_{k,t} > \amt_k \right\} \right|\right]
	\le \frac{\pi^2}{6 \nban}.
	\]
\end{lemma}

\begin{proof}
Fix some arm $k$.
Fix some $s,\zeta > 0$. Regard the inequality
\[
x \ge \frac{s}{\amt} - \sqrt{2 \frac{s}{\amt} \zeta},
\]
for all positive $x$ and $\amt$.
This is a quadratic inequality in the parameter $\sqrt{\frac{1}{\amt}}$,
which holds if and only if
\[
\sqrt{\frac{1}{\amt}} \le \sqrt{\frac{\zeta}{2s}} + \sqrt{\frac{\zeta}{2s} + \frac{x}{s}}.
\]
In particular, whenever $x \ge \frac{s}{\amt} - \sqrt{2 \frac{s}{\amt} \zeta}$,
it holds that 
\begin{equation} \label{eq:bad-slow-1}
\amt \ge \left( \sqrt{\frac{\zeta}{2s}} + \sqrt{\frac{\zeta}{2s} + \frac{x}{s}}\right)^{-2}.
\end{equation}

Define for any $t \ge 0$ and $\zeta > 0$,
\[
\lbs_{\zeta,k,t} = \begin{cases}
\left( \sqrt{\frac{\zeta}{2 \sas_{k,t}}} + \sqrt{\frac{\zeta}{2\sas_{k,t}} + \frac{\xs_{k,t}}{ \sas_{k,t}}} \right)^{-2}
& \sas_{k,t} > 0 \\
0 & \sas_{k,t} = 0
\end{cases}.
\]
Note that $\lbs_{\zeta_t,k,t} = \lbs_{k,t}$.

For any integer $s' \ge 0$, let $\rvt_{s'}$ be the first iteration $t$ that $\sas_{k,t} > s' \amt_k$.
From the way we extended the values of $\sas_{k,t}$ to $t>n$ in equation~\eqref{eq:gslow-sas}, it holds that for any $s'$, $\rvt_{s'}$ is bounded by a constant.
For all $t > 0$,
\begin{align}
\Pr\left[\lbs_{k,t} > \amt_k\right] 
&=\Pr\left[\lbs_{\zeta_t,k,t} > \amt_k\right] \nonumber\\
&\le \Pr\left[\exists i \ge 0 \colon \sas_{k,i} \le t \amt_k,\ \lbs_{\zeta_t,k,i} > \amt_k \right] \nonumber\\
&\le \sum_{s' = 0}^{t-1} \Pr\left[\exists i \ge 0 \colon s' \amt_k < \sas_{k,i} \le (s'+1) \amt_k,\ \lbs_{\zeta_t,k,i} > \amt_k \right] \nonumber \\
&= \sum_{s' = 0}^{t-1} \Pr\left[\exists i \ge 0 \colon s' \amt_k < \sas_{k,i} \le (s'+1) \amt_k,\ 
\left( \sqrt{\frac{\zeta_t}{2\sas_{k,i}}} + \sqrt{\frac{\zeta_t}{2\sas_{k,i}} + \frac{\xs_{k,i}}{\sas_{k,i}}}\right)^{-2} > \amt_k \right] \nonumber \\
&\le \sum_{s' = 0}^{t-1} \Pr\left[ 
\left( \sqrt{\frac{\zeta_t}{2(s'+1)\amt_k}} + \sqrt{\frac{\zeta_t}{2(s'+1)\amt_k} + \frac{\xs_{k,\rvt_{s'}}}{(s'+1)\amt_k}}\right)^{-2} > \amt_k \right] \nonumber \\
&\le \sum_{s' = 0}^{t-1} \Pr\left[ 
\xs_{k,\rvt_{s'}} < \frac{(s'+1)\amt_k}{\amt_k} - \sqrt{2 \frac{(s'+1)\amt_k}{\amt_k} \zeta_t} \right] \label{eq:bad-slow-2} \\
&\le \sum_{s' = 1}^{t-1} \Pr\left[\xs_{k,\rvt_{s'}} < s'\left(1 - \sqrt{2 \zeta_t / s'} + \frac{1}{s'}\right)\right] \nonumber \\
&\le \sum_{s' = 1}^{t-1} \exp\left( - \frac{1}{2} s' \left( \sqrt{2 \zeta_t / s'} - \frac{1}{s'} \right)^2 \right) \label{eq:bad-slow-3} \\
&\le \sum_{s' = 1}^{t-1} \exp\left( -\zeta_t + \sqrt{2 \zeta_t / s'}  \right) \nonumber \\
&\le t \exp\left( -\zeta_t + \sqrt{2 \zeta_t}  \right) \nonumber \\
&= t \exp\left( -\ln(1/\epsilon_t)  \right) \nonumber \\
&= t^{-2}\nban^{-1}. \nonumber
\end{align}
where inequality~\eqref{eq:bad-slow-2} follows from \eqref{eq:bad-slow-1} by substituting $s = (s'+1) \amt_k$, $\amt = \amt_k$, 
$x = \xs_{k,\rvt_{s'}}$ and $\zeta = \zeta_t$;
and inequality~\eqref{eq:bad-slow-3} follows from Lemma~\ref{lem:mart-sum}, by substituting $X_i = \xs_{k,i} - \xs_{k,i-1}$, $Y_i = Z_i$, $P_i = \frac{\sas_{k,i} - \sas_{k,i-1}}{\amt_k}$, $r = s'$ and $\delta = \sqrt{2 \zeta_t / s'} - \frac{1}{s'}$.

This implies that
\[
\E\left[\left| \left\{ 0 \le t \le \nit \colon \lbs_{k,t} > \amt_k \right\} \right|\right]
\le \sum_{t=1}^{\nit} \Pr[\lbs_{k,t} > \amt_k]
\le  \sum_{t=1}^{\nit} \frac{1}{t^2 \nban}
\le \sum_{t=1}^\infty \frac{1}{t^2 \nban}
= \frac{\pi^2}{6 \nban}.
\]
\end{proof}

To conclude the proof, note the following:
the expected number of iterations $t$ that there exists $k$ such that $\lb_{k,t} = 0$, is at most the expected number of iterations that there exists $k$ such that $\lbf_{k,t} = 0$. 
This quantity is bounded by $\E[\rvt_0] = O(\max(1, \log \frac{1}{\amt_1}))$, by Lemma~\ref{lem:first-bd}.

The expected number of iterations $t$ such that there exists $k$ that $\lb_{k,t} > \amt_k$, is bounded by the expected number of iterations that there exists $k$ that $\lbs_{k,t} > \amt_k$, which is bounded by a constant, from Lemma~\ref{lem:bad-slow-bd}.
This concludes the proof.

\subsection{Proof of Lemma~\ref{lem:slow}} \label{sec:p-slow}
We begin with a lemma:

\begin{lemma} \label{lem:slow-grow}
	Fix an integer $k$ and real numbers $a$ and $\amt$ such that $1 \le k \le \nban$, $0 < a \le 1$, and $\frac{1}{\amt} \ge \frac{1+a}{\amt_k}$.
	Then,
	\[
	\E\left[\sum_{t \colon \lb_{k,t-1} \le  \amt} \indi_{\alloc_{k,t}\le \lbf_{k,t-1}} \alloc_{k,t} \right]
	\le \amt \left(\frac{1029 \zeta_n}{a^2} \right).
	\]
\end{lemma}

\begin{proof}
	Assume that $\sas_{k,t}$ and $\xs_{k,t}$ are defined also for $t > n$, as defined in equation~\eqref{eq:gslow-sas}.
	Let $\alpha = \left( \frac{32}{a} \right)^2$,
	and let 
	\[
	s' = \left\lceil 1 + \alpha \zeta_n \right\rceil.
	\]
	First, for any $s \ge s'$, and for any $x \le s + \sqrt{3 s \log n}$, it holds that
	\begin{align}
	\left( \sqrt{\frac{\zeta_n}{2 (s-1)}} + \sqrt{\frac{\zeta_n}{2 (s-1)} + \frac{x}{s-1}} \right)^2
	&= \frac{s}{s-1} \left( \sqrt { \frac{\zeta_n}{2 s}} + \sqrt{\frac{\zeta_n}{2 s} + \frac{x}{s}} \right)^2 \nonumber\\
	&\le \frac{s}{s-1} \left( \sqrt { \frac{\zeta_n}{2 s}} + \sqrt{\frac{\zeta_n}{2 s} + 1 + \sqrt{\frac{3 \log n}{s}}} \right)^2 
	\nonumber\\
	&\le \frac{s}{s-1} \left( \sqrt { \frac{\zeta_n}{2 s}} + \sqrt{\frac{\zeta_n}{2 s} + 1 + \sqrt{\frac{\zeta_n}{s}}} \right)^2 
	\nonumber\\
	&\le \frac{s}{s-1} \left( \sqrt { \frac{\zeta_n}{2 s}} + \sqrt{\sqrt{\frac{\zeta_n}{2 s}} + 1 + \sqrt{\frac{\zeta_n}{s}}} \right)^2 
	\nonumber\\
	&\le \frac{s}{s-1} \left( \sqrt { \frac{\zeta_n}{2 s}} + \frac{1}{2} \sqrt{\frac{\zeta_n}{2 s}} + 				
	\frac{1}{2}\sqrt{\frac{\zeta_n}{s}} + 1 \right)^2 \label{eq:slowg-0}\\
	&\le \left( 1 + \frac{1}{s-1} \right) \left( 1 + 15 \sqrt { \frac{\zeta_n}{s}} \right) \nonumber \\
	&\le \left( 1 + \frac{1}{\alpha} \right) \left( 1 + 15 \sqrt { \frac{1}{\alpha}} \right)\nonumber \\
	&= 1 + a, \label{eq:slowg-1}
	\end{align}
	where inequality~\eqref{eq:slowg-0} follows from the fact that $\sqrt{1+x} \le 1 + \frac{x}{2}$ for $x \ge 0$.
	
	For any integer $s \ge s'$, let $\rvt_{s}$ be the last $t$ such that $\sas_{k,t} \le s \amt$, and $A_s$ be the event
	\[
	\left( \sqrt{\frac{\zeta_n}{2 (s-1) \amt}} + \sqrt{\frac{\zeta_n}{2 (s-1)\amt} + \frac{\xs_{k,\rvt_{s}}}{(s-1)\amt}} \right)^2
	\ge \frac{1}{\amt}. 
	\]
	Substituting $x = \xs_{k,\rvt_{s}}$ and $s = \frac{s}{1+a}$ in inequality~\eqref{eq:slowg-1},
	we obtain that whenever
	$\xs_{k,\rvt_{s}} \le \frac{s}{1+a} + \sqrt{3 \frac{s}{1+a} \log n}$,
	$A_s$ does not hold.
	
	Applying Lemma~\ref{lem:mart-sum} with $X_t = \xs_{k,t} - \xs_{k,t-1}$,
	$Y_t = Z_t$, $P_t = \frac{\sas_{k,t} - \sas_{k,t-1}}{\amt_k}$,
	$r = \frac{s}{1+a}$ and $\delta = \sqrt{\frac{3 (1+a) \log n}{s}}$, it holds that
	\[
	\Pr[A_s]
	\le \Pr \left[\xs_{k,\rvt_{s}} > \frac{s}{1+a} + \sqrt{3 \frac{s}{1+a} \log n}\right]
	\le \exp\left(-\frac{\delta^2 s}{3(1+a)}\right)
	= \frac{1}{n}.
	\]
	
	This suffices to complete the proof:
	\begin{align*}
	&\E\left[\sum_{1 \le t \le n \colon \frac{1}{\lb_{k,t-1}} \ge  \frac{1}{\amt}} \indi_{\alloc_{k,t}\le \lbf_{k,t-1}} \alloc_{k,t} \right] \\
	&= \E\left[\sum_{1 \le t \le \min(n,\rvt_{s'}+1) \colon \frac{1}{\lb_{k,t-1}} \ge \frac{1}{\amt}} \indi_{\alloc_{k,t}\le \lbf_{k,t-1}} \alloc_{k,t} \right]
	+ \E\left[\sum_{\min(n,\rvt_{s'}+1) < t \le n  \colon \frac{1}{\lb_{k,t-1}} \ge  \frac{1}{\amt}} \indi_{\alloc_{k,t}\le \lbf_{k,t-1}} \alloc_{k,t} \right] \\
	&\le \amt s' + \E\left[ \sum_{s=s'+1}^{n-1} 
	\sum_{\substack{t \colon \\ (s - 1)\amt_k < \sas_{k,t-1} \le s \amt_k}} \indi_{\alloc_{k,t}\le \lbf_{k,t-1}} \indi_{\frac{1}{\lb_{k,t-1}} \ge  \frac{1}{\amt}} \alloc_{k,t} \right] \\
	&\le \amt s' + \E\left[ \sum_{s=s'+1}^{n-1} 
	\sum_{\substack{t \colon \\ (s - 1)\amt_k < \sas_{k,t-1} \le s \amt_k}}
	\indi_{A_s}
	\indi_{\alloc_{k,t}\le \lbf_{k,t-1}} 
	\alloc_{k,t} \right] \\
	&\le \amt s' + \E\left[ \sum_{i=s'+1}^{n-1} 2\amt \cdot \indi_{A_s} \right] \\
	&\le \amt(s' + 2).
	\end{align*}
\end{proof}

Lemma~\ref{lem:below} implies that for any $j$,
$1 \le j \le \nban$, and for any $t \in T$,
if $\lbf_{k,t-1} > \amt_j$ then there are at least $j$ arms $i$ for which $\lb_{i,t-1} \le \amt_j < \lb_{k,t-1}$.
Therefore,
\begin{align}
\sum_{t \in T \colon k \in B_t \cup C_t}
&\begin{cases}
\min(\lbf_{k,t-1}, M_{k,t}) (1/\nu_{\nal+1}-1/\nu_{k}) & |A_t| = \nal \\
\min(\lbf_{k,t-1}, M_{k,t}) (1/\nu_{\nal}-1/\nu_{k}) & |A_t| < \nal
\end{cases} \nonumber\\
&\le \sum_{\substack{t \in T \colon k \in B_t \cup C_t \\ \lbf_{k,t-1} \le \amt_{\nal+1}}}
\min(\lbf_{k,t-1}, M_{k,t}) (1/\nu_{\nal+1}-1/\nu_{k}) \nonumber \\
&+ \sum_{\substack{t \in T \colon k \in B_t \cup C_t \\ \lbf_{k,t-1} \le \amt_{\nal}}}
\min(\lbf_{k,t-1}, M_{k,t}) (1/\nu_{\nal}-1/\nu_{k}). \label{eq:slow-0}
\end{align}

Take some $k' < k$.
Let $a = \min\left(1, \frac{\amt_k}{\amt_{k'}} - 1\right)$.
Then
\begin{align}
&\E \left[\sum_{\substack{t \in T \colon k \in B_t \cup C_t \\ \lb_{k,t-1} \le \amt_{k'}}}
\min(\lbf_{k,t-1}, M_{k,t}) \right] \\
&= \E \left[\sum_{\substack{t \in T \colon k \in B_t \\ \lb_{k,t-1} \le \amt_{k'}}}
\lbf_{k,t-1} \right]
+ \E \left[ \sum_{\substack{t \in T \colon k \in C_t \\ \lb_{k,t-1} \le \amt_{k'}}}
\indi_{\alloc_{k,t}\le \lbf_{k,t-1}} \alloc_{k,t} \right] \nonumber\\
&= 2 \E \left[ \sum_{\substack{t \in T \colon k \in B_t \\ \lb_{k,t-1} \le \amt_{k'}}}
\indi_{\alloc_{k,t}= \lbf_{k,t-1}} \lbf_{k,t-1}\right]
+ \E \left[ \sum_{\substack{t \in T \colon k \in C_t \\ \lb_{k,t-1} \le \amt_{k'}}}
\indi_{\alloc_{k,t}\le \lbf_{k,t-1}} \alloc_{k,t} \right] \label{eq:slow-1} \\
&= 2 \E \left[ \sum_{\substack{t \in T \colon k \in B_t \\ \lb_{k,t-1} \le \amt_{k'}}}
\indi_{\alloc_{k,t}\le \lbf_{k,t-1}} \alloc_{k,t}\right]
+ \E \left[ \sum_{\substack{t \in T \colon k \in C_t \\ \lb_{k,t-1} \le \amt_{k'}}}
\indi_{\alloc_{k,t}\le \lbf_{k,t-1}} \alloc_{k,t} \right] \label{eq:slow-11} \\
&\le 2 \E \left[ \sum_{\substack{t \colon 1 \le t \le n \\ \lb_{k,t-1} \le \amt_{k'}}}
\indi_{\alloc_{k,t}\le \lbf_{k,t-1}} \alloc_{k,t} \right] \nonumber \\
&\le 2  \amt_{k'}(\frac{1029 \zeta_n}{a^2}), \label{eq:slow-2}
\end{align}
where inequality~\eqref{eq:slow-1} follows from the fact that if $k \in B_t$ then $k$ is allocated according to case $B$, and the probability that $M_{k,t} = \lbf_{k,t-1}$ conditioned on $k \in B_t$ is $1/2$, independently on $\lbf_{k,t-1}$; \eqref{eq:slow-11} follows from the fact that whenever $k \in B_t$, it never holds that $\alloc_{k,t}< \lbf_{k,t-1}$; 
and inequality~\eqref{eq:slow-2} follows from Lemma~\ref{lem:slow-grow}.

If $a < 1$ this implies that
\begin{align*}
\E \left[ \sum_{\substack{t \in T \colon k \in B_t \cup C_t \\ \lbf_{k,t-1} \le \amt_{k'}}}
\min(\lbf_{k,t-1}, M_{k,t}) (1/\nu_{k'}-1/\nu_{k}) \right]
&\le \amt_{k'}(\frac{2058 \zeta_n}{a^2})(1/\amt_{k'}-1/\amt_{k}) \\
&\le \amt_{k}(\frac{2058 \zeta_n}{a^2})(1/\amt_{k'}-1/\amt_{k}) \\
&= \frac{2058 \zeta_n}{a} \\
&= \frac{2058 \zeta_n}{\amt_k/\amt_{k'}-1}.
\end{align*}

If $a = 1$, then
\begin{align*}
\E \left[\sum_{\substack{t \in T \colon k \in B_t \cup C_t \\ \lbf_{k,t-1} \le \amt_{k'}}}
\min(\lbf_{k,t-1}, M_{k,t}) (1/\nu_{k'}-1/\nu_{k}) \right]
&\le \amt_{k'}(\frac{2058 \zeta_n}{a^2})(1/\amt_{k'}-1/\amt_{k}) \\
&= 2058 \zeta_n(1-\frac{\amt_{k'}}{\amt_{k}})\\
&\le 2058 \zeta_n.
\end{align*}

Therefore, for any value of $a$,
\begin{align*}
\E \left[ \sum_{\substack{t \in T \colon k \in B_t \cup C_t \\ \lbf_{k,t-1} \le \amt_{k'}}}
\min(\lbf_{k,t-1}, M_{k,t}) (1/\nu_{k'}-1/\nu_{k}) \right]
&\le \frac{2058 \zeta_n}{\amt_k/\amt_{k'}-1} + 2058 \zeta_n \\
&= 2058 \zeta_n \frac{\amt_k}{\amt_k - \amt_{k'}}.
\end{align*}

This, together with inequality~\eqref{eq:slow-0}, conclude the proof.

\section{Proof of Theorem~\ref{thm:lb-main}} \label{sec:pr:lb}

Take an algorithm $A$, and we can assume that it is deterministic, since we are bounding an expected regret over all inputs.
Notice that the optimal allocation strategy is to fully allocate all the arms $1, \dots, r$, and additionally allocate some of the arms $r+1, \dots, 2r$.
Therefore, the expected regret on iteration $t$ satisfies
\begin{align} 
\E[\regret_t \mid \alloc_{1,t} \cdots \alloc_{\nban,t}]
&\ge \sum_{k \le r \colon \alloc_{k,t} > \amt_k} |\alloc_{k,t} - \amt_k| \frac{r}{2}
+ \sum_{k \le r \colon \alloc_{k,t} < \amt_k} |\amt_k - \alloc_{k,t}| \left(\frac{1}{\amt_k} - \frac{r}{2}\right) \nonumber \\
&\ge \frac{r}{2} \sum_{k = 1}^r |\amt_k - \alloc_{k,t}|, \label{eq:lbm-0}
\end{align}
where the sum over $\alloc_{k,t} > \amt_k$ corresponds to over-allocation of arms $k \le r$, and the sum over $\alloc_{k,t} > \amt_k$ corresponds to under-allocation of these arms.

The idea of the proof is to show that on any iteration $t$, and for any arm $k \le r$, the algorithm cannot estimate the value of $\amt_k$ with an error lower than $\Omega(1/(rt))$, therefore, the expected value of $|\amt_k - \alloc_{k,t}|$ will be $\Omega(1/(rt))$, and \eqref{eq:lbm-0} will imply that the regret on iteration $t$ will be $\Omega(r/t)$.

We start by giving a definition of a distance between two distributions.
Let $\Omega $ be a finite sample space, and $\mu_1, \mu_2$ be distribution measures over $\Omega$.
The \emph{total variation distance} between $\mu_1$ and $\mu_2$ is defined as
\[
d(\mu_1, \mu_2)
= \frac{1}{2} \sum_{\omega \in \Omega} | \mu_1(\omega) - \mu_2(\omega)|
= \max_{S \subseteq \Omega} |\mu_1(S) - \mu_2(S)|.
\]
This distance is subadditive in terms of a Cartesian product, as stated in the following lemma.

\begin{lemma} \label{lem:statistical-dist-marginals}
	Let $\Omega_1, \dots, \Omega_t$ be sample spaces,
	and let $\Omega = \Omega_1 \times \cdots \times \Omega_t$.
	Let $\mu$ and $\eta$ be measures over $\Omega$,
	and let $\mu_i$ and $\eta_i$ be the $\Omega_i$-marginals of $\mu$ and $\eta$ respectively, for all $1 \le i \le t$.
	Fix an $\epsilon \ge 0$. Assume that for any $1 \le i \le t$, and for any
	$\omega_1 \in \Omega_1, \dots, \omega_{i-1} \in \Omega_{i-1}$,
	$(\mu_i | \omega_1, \dots, \omega_{i-1})$ and $(\eta_i | \omega_1, \dots, \omega_{i-1})$ have a distance of at most $\epsilon$,
	where $(\mu_i | \omega_1, \dots, \omega_{i-1})$ is $\mu_i$ conditioned on $\omega_1, \dots, \omega_{i-1}$,
	and $(\eta_i | \omega_1, \dots, \omega_{i-1})$ is defined similarly.
	Then, $d(\mu , \eta) \le t \epsilon$.
\end{lemma}

Additionally, if $f \colon \Omega \to \mathbb{R}$ is a function, and $\mu_1,\mu_2$ are measures over $\Omega$, we can bound $\E_{\mu_1} f - \E_{\mu_2} f$ in terms of $d(\mu_1, \mu_2)$, as described in the following lemma:

\begin{lemma} \label{lem:func-dist}
	Let $\Omega$ be a sample space, let $a > 0$, let $f \colon \Omega \to [0, a]$ and let $\mu_1, \mu_2$ be measures over $\omega$. Then
	\[
	\E_{\omega \sim \mu_1} f(\omega) - \E_{\omega \sim \mu_2} f(\omega) \le a d(\mu_1, \mu_2).
	\]
\end{lemma}

\begin{proof}
	It holds that 
	\[
	\E_{\mu_1} f(\omega) - \E_{\mu_2} f(\omega) 
	= \sum_{\omega \in \Omega} (\mu_1(\omega) - \mu_2(\omega)) f(\omega)
	\le \sum_{\omega \colon \mu_1(\omega) \ge \mu_2(\omega)} (\mu_1(\omega) - \mu_2(\omega)) a
	\le a d(\mu_1, \mu_2).
	\]
\end{proof}

Fix $0 \le t \le \nit$, $k \le r$, and let $\Omega = \{0,1\}^{t \nban}$ be a sample space that contains vectors $(x_{k,i})_{1 \le k \le \nban,\ 1 \le i \le t}$.
Given two values $\frac{1}{2r} \le a < b \le \frac{1}{r}$,
let $\mu$ be a distribution over $\Omega$, which equals the distribution over $(\scc_{k,i})_{1 \le k \le \nban,\ 1 \le i \le t}$ when $\amt$ is drawn from $(\mathcal{D} \mid \amt_k = a)$. Formally, for any $x \in \Omega$,
\begin{equation} \label{eq:lb-mudec}
\mu(x) 
= \Pr_{\amt \sim \mathcal{D}}[\scc_{k,i} = x_{k,i} \text{ for all } 1 \le i \le t, 1 \le k \le \nban \mid \amt_k = a].
\end{equation}
Similarly, let $\eta$ be the corresponding distribution conditioned on $\amt_k = b$.
We can apply Lemma~\ref{lem:statistical-dist-marginals} by substituting $t=t$, $\epsilon = \left( 1 - \frac{a}{b} \right)$, and substituting $\Omega_i$ with the marginal of $\Omega$ on the coordinates $\{ (k,i) \colon 1 \le k \le \nban \}$, for all $1 \le i \le t$. The lemma implies that $d(\mu, \eta) \le t \left( 1 - \frac{a}{b} \right)$, and this quantity is at most $2tr(b-a)$.
Note that for any $x \in \Omega$, the value of $\alloc_{k,t+1}$ is deterministically defined given that $x$ occurs.
Therefore, we can define a function $f \colon \Omega \to [0, b-a]$ by
\[
f(x) = \begin{cases}
0 & \text{if  $\alloc_{k,t+1} < a$ given $x$} \\
\alpha & \text{if $\alloc_{k,t+1} = a + \alpha$ given $x$, for some $0 \le \alpha \le b-a$} \\
b-a & \text{if  $\alloc_{k,t+1} > b$ given $x$}
\end{cases}.
\]
Lemma~\ref{lem:func-dist} implies that 
\begin{align*}
\E_{\mu}[|\alloc_{k,t+1} - \amt_k|] + \E_{\eta}[|\alloc_{k,t+1} - \amt_k|]
&= \E_{\mu}[|\alloc_{k,t+1} - a|] + \E_{\eta}[|\alloc_{k,t+1} - b|] \\
&\ge \E_{x \sim \mu} f(x) + \left( (b-a) - \E_{x \sim \eta} f(x) \right) \\
&\ge (b-a) - (b-a) d(\mu, \eta) \\
&\ge (b-a)(1 - 2 t r (b-a)).
\end{align*}
Therefore, for any $t \ge 1$,
\begin{align*}
&2 \E_{\amt \sim \mathcal{D}}[|\alloc_{k,t+1} - \amt_k|] \\
&= 2 \cdot 2r \int_{a = \frac{1}{2r}}^{\frac{1}{r}} \E_{\amt \sim \mathcal{D}}[|\alloc_{k,t+1} - \amt_k| \mid \amt_k = a] da\\
&\ge 2r \int_{a = \frac{1}{2r}}^{\frac{1}{r} - \frac{1}{4rt}} 
\left( \E_{\amt \sim \mathcal{D}}[|\alloc_{k,t+1} - \amt_k| \mid \amt_k = a]
+ \E_{\amt \sim \mathcal{D}}\left[|\alloc_{k,t+1} - \amt_k| \middle\vert \amt_k = a + \frac{1}{4rt}\right]\right) da\\
&\ge 2r \int_{a = \frac{1}{2r}}^{\frac{1}{r} - \frac{1}{4rt}} \frac{1}{4rt}\left(1 - \frac{2 r t}{4rt} \right) da \\
&= \frac{1}{8 r t} - \frac{1}{16 r t^2}.
\end{align*}
Combining with inequality~\eqref{eq:lbm-0}, this implies that
\[
\E R^{(n)}
\ge \sum_{t=1}^{n-1} \frac{r}{32 t} - \sum_{t=1}^{n-1} \frac{r}{64 t^2}
\ge \frac{r}{32} \left( H(n-1) - \frac{\pi^2}{12} \right).
\]

\section{Regret Lower Bound with respect to the parameters $\amt_1 \cdots \amt_{\nban}$}\label{sec:lb-nu}

\begin{theorem} \label{thm:lb-rem}
	Fix integers $\nban$ and $\nal$ such that $\nban > \nal + 1$, and fix numbers $v_1, \dots, v_\nal$, such that 
	$v_1 + \cdots + v_{\nal} < 1$.
	Let $B$ be the set of all vectors $\amt = (\amt_1, \dots,
	\amt_\nban) \in \mathbb{R}^\nban$, such that: 
	(1) For all $k \le \nal$, it holds that $\amt_k =
	v_k$, and 
	(2) For all $\nal < k \le \nban$, it holds that
	$\amt_k > 1$. 
	For any $\amt \in B$, define $\amt^* = \min_{\nal + 1 \le k \le \nban} \amt_k$.
	Additionally, define 
	\(
	D(p || q) = p \log \frac{p}{q} + (1-p) \log \frac{1-p}{1-q}.
	\) 
	Assume an anytime algorithm $A$, such that for all $\amt \in B$, and for all $a > 0$, 
	$\lim_{n \to \infty} \E R^{(n)}(A, \amt)/n^a = 0$. 
	Define 	$C(a) = \max \left( \frac{4(1-a^{-1})^2}{4(1-a^{-1})^2 + 1},\,
	\frac{a^{-1}}{- 4\log (1-a^{-1})} \right)$. 
	Then, for all $\amt \in B$, 
	\begin{equation}
	\liminf_{n \to \infty} \frac{\E R^{(n)}(A, \amt)}{\log n} 
	\ge \!\!\! \sum_{k \colon \nal + 1 \le k \le \nban,\ \amt_k > \amt^*} \!\!\!\frac{1/\amt^* - 1/\amt_k}{D(1/\amt_k || 1/\amt^*)} 
	\ge  C(\amt^*)
	\!\!\!\sum_{k \colon \nal + 1 \le k \le \nban,\ \amt_k \ne \amt^*} \!\!\!\frac{\amt_k}{\amt_k - \amt^*}~. \label{eq:lb-rem-main-0}
	\end{equation}
	
\end{theorem}

The proof follows the same steps taken in the proof of Theorem 2 in the paper by \cite{LR},
yet it is simpler to rewrite it instead of stating all the differences.
All asymptotic notations correspond only to $n$, and consider the other parameters of the problem as constants.

For any $k \ge \nal + 1$, and any integer $n$,
let $T_n(k)$ be the random variable which equals $\sum_{t=1}^n \alloc_{k,t}$.
It holds 
\begin{equation} \label{eq:lb1-0}
\E R^{(n)}(A,\amt) \ge \sum_{k \colon \amt_k > \amt^*} \E T_n(k) \left( \frac{1}{\amt^*} - \frac{1}{\amt_k} \right).
\end{equation}

Fix some $k$ such that $\amt_k > \amt^*$, and we will prove that 
\begin{equation} \label{eq:lb-rem-5}
\liminf_{n \to \infty} \frac{\E T_n(k)}{\log n}
\ge \frac{1}{D(1/\amt^k || 1/\amt^*)},
\end{equation}
and this, together with inequality~\eqref{eq:lb1-0} completes the
proof of the left inequality~\eqref{eq:lb-rem-main-0}.
Let $\theta_k = \frac{1}{\amt_k}$, and let $\theta^* = \frac{1}{\amt^*}$.
Fix any $\delta > 0$.
Fix some $\lambda$ such that $\theta^* < \lambda$ and $| D(\theta_k || \lambda) - D(\theta_k || \theta^*) | < \delta D(\theta_k || \theta^*)$.
Let $\gamma \in \mathbb{R}^\nban$ be a vector defined as
\[
\gamma_i = \begin{cases}
\frac{1}{\lambda} & i = k \\
\amt_i & i \ne k
\end{cases}.
\]
Fix $a$, $0 < a < \delta$. It holds that
\begin{equation} \label{eq:lb-rem-0}
(n - O(\log n)) P_{\gamma}\left[T_n(k) < \frac{(1-\delta) \log n}{D(\theta_k || \lambda)}\right]
\le \E_\gamma(n - T_n(k))
= o(n^a).
\end{equation}
Given a value of $\theta$, and an integer $t$, let $Z_{\theta, t}$ be the random variable which equals $\alloc_{k,t} \theta$ if $\scc_{k,t} = 1$ and $1 - \alloc_{k,t} \theta$ otherwise. Namely, $Z_{\theta,t}$ is the probability that $\scc_{k,t}$ had to get its value given $\alloc_{k,t}$ and the parameter of arm $k$.
Let
\[
L_n = \sum_{t=1}^n \log \frac {Z_{\theta_k, t}}{Z_{\lambda, t}}.
\]
Let 
\[
C_n = \left\{ T_n(k) < \frac{(1- \delta) \log n}{D(\theta_k || \lambda)},\, L_n \le (1 - a) \log n \right\}.
\]
It follows from \eqref{eq:lb-rem-0} that 
\begin{equation} \label{eq:lb-rem-1}
\Pr_\gamma(C_n) = o(n^{a-1}).
\end{equation}

Note that for any $r \ge 0$,
\begin{align} 
\Pr_{\gamma} \left[ T_n(k) = r, L_n \le (1 - a) \log n \right]
&= \int_{(T_n(k) = r, L_n \le (1 - a) \log n)} \prod_{t=1}^n \frac{Z_{\lambda, t}}{Z_{\theta_k, t}} d P_\amt \nonumber \\
&\ge \exp(-(1-a) \log n) \Pr_\amt \left[ T_n(k) = r, L_n \le (1 - a) \log n \right]. \label{eq:lb-rem-2}
\end{align}
Since $C_n$ is a disjoint union of events of the form $\left\{T_n(k) = r, L_n \le (1 - a) \log n\right\}$,
with $r < (1 - \delta) \log n/ D(\theta_k || \lambda)$, it follows from \eqref{eq:lb-rem-1} and \eqref{eq:lb-rem-2} that
\begin{equation} \label{eq:lb-rem-3}
\lim_{n \to \infty} \Pr_{\amt}(C_n) 
\le \lim_{n \to \infty} n^{1-a} \Pr_{\gamma}(C_n) 
= 0.
\end{equation}

Let $\rvt$ be the first $t$ such that $T_t(k) \ge \frac{(1- \delta) \log n}{D(\theta_k || \lambda)}$.
The inequality $D(\epsilon p || \epsilon q) \le \epsilon D(p || q)$ for all $0 < p,q,\err < 1$, implies that 
$\E L_{\rvt} \le \log n +O(1)$.
Therefore, using standard concentration bounds, it holds that
\begin{equation} \label{eq:lb-rem-4}
\lim_{n \to \infty}
\Pr_\amt \left[ \exists i < \rvt,\, L_i > (1 - a) \log n \right]
= 0.
\end{equation}

From \eqref{eq:lb-rem-3} and \eqref{eq:lb-rem-4} we see that
\[
\lim_{n\to \infty} \Pr_\amt \left[ T_n(k) < \frac{(1 - \delta) \log n}{(1+\delta) D(\theta_k, \theta^*)} \right]
\le \lim_{n\to \infty} \Pr_\amt \left[ T_n(k) < \frac{(1 - \delta) \log n}{D(\theta_k, \lambda)} \right]
= 0,
\]
from which \eqref{eq:lb-rem-5} follows.
This concludes the proof of the left inequality~\eqref{eq:lb-rem-main-0}.

Next, we prove the right inequality~\eqref{eq:lb-rem-main-0}.
Fix $0 < p < q < 1$, and let $\epsilon$ be a number such that $q = (1+\epsilon) p$.
Assume that $\epsilon \le 1$.
Estimating the Taylor sum of $D(p || p(1+\epsilon))$ around $\epsilon = 0$, 
we get that there exists $0 \le \zeta \le \epsilon$ such that
\begin{align*}
D(p || q)
&= \left.\frac{\partial^2 D(p || (1+\epsilon) p)}{\partial \epsilon^2}\right\vert_{p = \zeta} \frac{\epsilon^2}{2} \\
&= \left( \frac{p}{(1 + \zeta)^2} + \frac{p^2 (1-p)}{(1- p - \zeta p)^2} \right) \frac{\epsilon^2}{2}\\
&\le \left( p + \frac{p^2 (1-p)}{(1-q)^2} \right) \frac{\epsilon^2}{2}\\
&\le \left( 1 + \frac{1}{4(1-q)^2} \right) \frac{\epsilon^2 p}{2}.
\end{align*}
This implies that,
\begin{align} \label{eq:lb-rem-6}
\frac{q-p}{D(p||q)}
\ge \frac{2}{\epsilon \left( 1 + \frac{1}{4(1-q)^2} \right)}
= \frac{8(1-q)^2}{\epsilon \left( 4(1-q)^2 + 1 \right)}
\ge \frac{4(1-q)^2}{ 4(1-q)^2 + 1} \frac{1+\epsilon}{\epsilon}
= \frac{4(1-q)^2}{4(1-q)^2 + 1} \frac{1/p}{1/p-1/q}.
\end{align}
Next, assume that $\epsilon > 1$. It holds that
\[
D(p || q)
\le (1-p) \log \frac{1-p}{1-q}
\le \log \frac{1}{1-q}.
\]
Therefore, 
\begin{align} \label{eq:lb-rem-7}
\frac{q-p}{D(p||q)}
\ge \frac{q^2}{4(q-p) D(p||q)}
\ge \frac{q}{- 4\log (1-q)} \frac{q}{q-p}
= \frac{q}{- 4\log (1-q)} \frac{1/p}{1/p-1/q}.
\end{align}
Inequalities \eqref{eq:lb-rem-6} and \eqref{eq:lb-rem-7} conclude the
proof of the right inequality~\eqref{eq:lb-rem-main-0},
replacing $p = \frac{1}{\amt_k}$ and $q = \frac{1}{\amt_{\nal+1}}$.

\end{document}